\pdfoutput=1
\documentclass{article}
\PassOptionsToPackage{numbers}{natbib}
\usepackage[preprint]{neurips_2023}

\PassOptionsToPackage{numbers}{natbib}

\usepackage[utf8]{inputenc} 
\usepackage[T1]{fontenc}    
\usepackage{hyperref}
\usepackage{url}            
\usepackage{booktabs}       
\usepackage{amsfonts}       
\usepackage{nicefrac}       
\usepackage{microtype}      
\usepackage{xcolor}         

\usepackage{amsmath}
\usepackage{amssymb}
\usepackage{mathtools}
\usepackage{amsthm}

\usepackage[capitalise]{cleveref}

\usepackage{bm,bbm}
\usepackage{xspace}
\usepackage{multirow}
\usepackage{cancel}
\usepackage[normalem]{ulem}
\usepackage{enumitem}
\usepackage{caption}
\usepackage{booktabs}

\usepackage{bm}

\usepackage{subcaption}
\usepackage{graphicx}
\usepackage{threeparttable, array, float}

\usepackage{algorithm}
\usepackage[noend]{algpseudocode}

\algnewcommand{\LeftComment}[1]{\Statex \(\triangleright\) \texttt{#1}}
\algnewcommand{\RightComment}[1]{\Statex \leavevmode\hfill\(\triangleright\) \texttt{#1}}

\newtheorem{theorem}{\bf Theorem}
\newtheorem{lemma}{\bf Lemma}
\newtheorem{proposition}{Proposition}

\newtheorem{condition}{\bf Condition}

\Crefname{condition}{Condition}{Conditions}
\Crefname{proposition}{Proposition}{Proposition}

\usepackage{letltxmacro}
\LetLtxMacro{\originaleqref}{\eqref}
\renewcommand{\eqref}{Eq.~\originaleqref}

\DeclareMathOperator*{\argmax}{arg\,max}
\DeclareMathOperator*{\argmin}{arg\,min}

\DeclarePairedDelimiter\abs{\lvert}{\rvert}%
\DeclarePairedDelimiter{\ceil}{\lceil}{\rceil}

\newcommand{\UCB}{\text{UCB}}

\newcommand{\upbra}[1]{^{(#1)}}

\renewcommand{\ge}{\geqslant}
\renewcommand{\le}{\leqslant}

\newcommand{\matob}{\texttt{CMA2B}\xspace}

\newcommand{\aas}{{AAS}\xspace}
\newcommand{\tas}{{TAS}\xspace}

\usepackage[colorinlistoftodos,prependcaption,textsize=tiny,textwidth=0.15\columnwidth]{todonotes}

\newcommand{\compilefullversion}{true} 
\ifthenelse{\equal{\compilefullversion}{false}}{%
	\newcommand{\OnlyInFull}[1]{}
	\newcommand{\OnlyInShort}[1]{#1}
}{
	\newcommand{\OnlyInFull}[1]{#1}%
	\newcommand{\OnlyInShort}[1]{}%
}

\usepackage{fontawesome5}

\usepackage[switch]{lineno}

\title{Adversarial Attacks on Cooperative Multi-agent Bandits}
\author{
Jinhang Zuo$^{1,2}$\thanks{Equal contribution.} \quad Zhiyao Zhang$^{3*}$ \quad Xuchuang Wang$^{4*}$ \quad Cheng Chen$^{5}$ \quad\textbf{Shuai Li}$^{6}$\\ \textbf{John C.S. Lui}$^{4}$ \quad
\textbf{Mohammad Hajiesmaili}$^{2}$ \quad \textbf{Adam Wierman}$^{1}$\\
$^{1}$California Institute of Technolog\quad
$^{2}$University of Massachusetts Amherst\\
$^{3}$Southeast University\quad
$^{4}$The Chinese University of Hong Kong\\
$^{5}$East China Normal University\quad
$^{6}$Shanghai Jiao Tong University
}

\begin{document}

\maketitle
\begin{abstract}
Cooperative multi-agent multi-armed bandits (\matob) consider the collaborative efforts of multiple agents in a shared multi-armed bandit game. We study latent vulnerabilities exposed by this collaboration and consider adversarial attacks on a few agents with the goal of influencing the decisions of the rest. More specifically, we study adversarial attacks on \matob in both homogeneous settings, where agents operate with the same arm set, and heterogeneous settings, where agents have distinct arm sets. In the homogeneous setting, we propose attack strategies that, by targeting just one agent, convince all agents to select a particular target arm $T-o(T)$ times while incurring $o(T)$ attack costs in $T$ rounds. In the heterogeneous setting, we prove that a target arm attack requires linear attack costs and propose attack strategies that can force a maximum number of agents to suffer linear regrets while incurring sublinear costs and only manipulating the observations of a few target agents. Numerical experiments validate the effectiveness of our proposed attack strategies.
\end{abstract}

\section{Introduction}\label{sec:intro}
Cooperative multi-agent multi-armed bandits (\matob) have been widely studied in recent years~\citep{bistritz2018distributed,boursier2019sic,yang2021cooperative,wang2023achieve}. In \matob, \(M\in\mathbb{N}^+\) agents cooperatively play multi-armed bandits with \(K\in\mathbb{N}^+\) arms in a sequential manner. 
In each decision round, each agent picks one arm and observes a reward sample drawn from a stochastic distribution associated with the pulled arm.
Their cooperative objective is to maximize their total cumulative rewards in \(T\in\mathbb{N}^+\) decision rounds, or minimize their total regret---the difference between the total expected rewards of all agents constantly pulling the optimal arm and the actual expected rewards of the considered algorithm.

Leveraging cooperation between agents, \matob algorithms can achieve an improved total regret of $O(K \log T)$~\citep{wang2023achieve}, compared with a total regret of $O(MK \log T)$ if no cooperation is used.
However, a security caveat arises when some agents occasionally get unreliable observations that may have been tampered with by malicious attackers. This concern becomes more serious in large-scale multi-agent systems, where assuring consistent reliability of every agent's observations becomes increasingly intricate. 
Given the collaborative nature of \matob, such adversarial attacks have the potential to not only influence the performance of their target agents, but also affect other agents in the same learning system.

Adversarial attacks on single-agent bandits have been recently studied in~\citep{jun2018adversarial,liu2019data}. In single-agent bandits, a successful attack means convincing the agent to pull a target arm a nearly linear number of times (i.e., \(T-o(T)\)) via manipulation of the agent's reward observations while only incurring a sublinear attack cost (i.e., \(o(T)\)). In contrast, in the multi-agent settings, the definition of a successful attack may vary. It might involve misleading a single agent or all agents, and the manipulations could target one agent or span across all agents. In this paper, we consider a challenging attack objective: \textit{affecting the maximum number of agents via attacking the reward observations of only a small subset of agents}.

Pursuing this objective, we discuss adversarial attacks in both homogeneous and heterogeneous \matob contexts. In the homogeneous settings, where each agent has access to all $K$ arms, it is feasible to convince all agents to choose a specific target arm in most time slots (if the target is a suboptimal arm, agents suffer linear regret). However, the number of agents that must be attacked and the overall cost of these attacks remain uncertain (addressed in Section~\ref{sec:homo}). Conversely, in the heterogeneous settings, where agents have different subsets of arms, the goal of directing all agents to select a single target arm becomes unattainable, especially if some agents lack access to this arm. As we elaborate in \Cref{sec:hetero}, even the task of convincing a subset of agents to select a target arm could require linear attack costs. Consequently, we shift our attack objective towards inducing the greatest number of agents to suffer linear regret.  This objective leads to three new challenges. First, determining the largest group of affected agents experiencing linear regret while incurring only sublinear costs remains an unresolved issue. Second, it is unclear how to select a small number of target agents for attack. Third, for these chosen target agents, we need to design attack strategies that can effectively influence other agents. 

\textbf{Our Contributions}.
In this paper, we provide an in-depth study of adversarial attacks on \matob. In the homogeneous setting, we propose attack strategies that can convince all agents to select a designated target arm a linear number of times by attacking only a single agent with sublinear attack costs, revealing the inherent vulnerability of homogeneous \matob algorithms.  In the heterogeneous settings, we demonstrate that the target arm attack may demand linear attack costs, and propose attack strategies that can compel a significant number of agents to experience linear regret by manipulating the observations of only a few agents with sublinear attack costs.
Our technical contributions are summarized below.

\begin{itemize}
    \item In~\Cref{sec:homo}, for the homogeneous \matob setting, we devise attack strategies targeting a \emph{single} agent to effectively misguide \emph{all} agents and prove their attack costs are $O(K \log T)$, independent of the number of agents $M$.
    \item In~\Cref{subsec:conflict}, for the heterogeneous \matob setting, after illustrating several examples in which linear cost is necessary to fool the agents into suffering linear regret, we provide a criterion to determine whether two agents can be simultaneously misled with only sublinear attack costs.
    \item In~\Cref{subsec:oracle-attack}, we extend the above criterion to the Affected Agents Selection (AAS) algorithm that identifies the largest set of agents eligible to be affected, ensuring a $(1-1/e)-$approximation guarantee. Furthermore, we design the Target Agents Selection (TAS) algorithm to choose a small subset of agents for attack (target agents) that can influence all agents selected by AAS.
    \item Based on the AAS and TAS algorithms, we propose the Oracle Attack (OA) strategy tailored for known heterogeneous \matob environments, accompanied by an analysis of the associated attack costs.
    In~\Cref{subsec:learning-then-attack}, we then adapt this strategy to unknown environments via the Learning-Then-Attack (LTA) strategy and offer its cost analysis. Both strategies retain the costs of $O(K \log T)$ as those in homogeneous settings. 
\end{itemize}
In addition to the algorithmic and theoretical contributions, we conduct experiments to evaluate our proposed attack strategies. Due to the space limit, proofs and empirical results are included in the appendix.

\textbf{Related Work}.
There is a large literature of work focused on \matob, e.g.,~\citep{rosenski2016multi,bistritz2018distributed,boursier2019sic,wang2020optimal,wang2020distributed,shi2021heterogeneous,yang2021cooperative,yang2022distributed,wang2022multiple,wang2023achieve,wang2023explore} and the references therein.  However, only a few works have studied misinformation in \matob learning, i.e., ~\citep{boursier2020selfish,vial2021robust,madhushani2021one,dubey2020private}.
In these scenarios, there are either malicious/selfish agents (e.g., byzantine agents~\citep{dubey2020private}) sharing wrong information (e.g., false arm recommendation~\citep{vial2021robust}, wrong reward observations~\citep{boursier2020selfish}), or imperfect communication (e.g., adversarial corruption~\citep{madhushani2021one}), resulting in other agents failing to find the optimal arm.
Our work is the first to study how an attacker may manipulate multi-agent cooperative learning.

Regarding adversarial attacks on single-agent multi-armed bandit problems, a growing literature focuses on a variety of bandit settings, e.g., ~\citep{jun2018adversarial, garcelon2020adversarial, ma2023adversarial, zuo2023adversarial, liu2020action, wang2023adversarial}. Specifically, \citep{jun2018adversarial} pioneered the formulation of adversarial attack models for stochastic bandits. \citep{garcelon2020adversarial} delved into attacks on linear contextual bandits, while \citep{ma2023adversarial} examined adversarial bandits, providing a lower bound of cumulative costs in such environments. Notably, these investigations do not extend to multi-agent bandit scenarios. To the best of our knowledge, our research marks the first exploration of adversarial attacks on \matob.

\section{Preliminaries}\label{sec:model}
We consider a \matob consisting of \(K\in\mathbb{N}^+\) arms, denoted by an arm set $\mathcal{K} \coloneqq \{1,2,\cdots,K\}$,
and \(M\in\mathbb{N}^+\) agents, denoted by an agent set $\mathcal{M} \coloneqq \{1,2,\cdots,M\}$.
Each arm \(k\in\mathcal{K}\) is associated with a $\sigma^2$-sub-Gaussian reward distribution and \emph{unknown} mean $\mu(k)\in [0, b]$ where \(b \in \mathbb{R}^+\) is known a prior. 
We assume $\mu(1) > \mu(2) > \cdots > \mu(K)$ such that the minimal mean reward gap is positive, which was also adopted in prior multi-agent bandits studies (e.g., \citep{rosenski2016multi}).
Assume there are \(T\in\mathbb{N}^+\) decision rounds for \matob, denoted the round set by \(\mathcal{T} \coloneqq \{1,2,\dots,T\}.\) We consider both homogeneous and heterogeneous settings for cooperative multi-agent bandits.

\textbf{Homogeneous settings}. In the homogeneous setting, each agent can equally observe and select every arm in $\mathcal{K}$. In round \(t\in\mathcal{T}\), each agent selects an arm $k_t\upbra{m} \in \mathcal{K}$ and observes a reward $X_t\upbra{m,0}(k_t\upbra{m})$ with expectation $\mu(k_t\upbra{m})$, where the superscript \(\upbra{m,0}\) refers to the vanilla (pre-attack) reward observation on agent \(m\). These agents share their information about arms with each other. We use regret to measure the performance of a policy, defined as
\begin{equation}\label{eq:R_Homo}\textstyle
	R(T) \coloneqq MT\mu(1) - \sum_{m=1}^M\sum_{t=1}^T\mu(k_t\upbra{m}),
\end{equation}
which is the difference between the maximized accumulative reward (all agents keep pulling optimal arm \(1\)) and the concerned policy's total reward.
Note that we do not consider the collision (e.g.,~\citep{boursier2019sic}) here, which means different agents can select the same arm in the same round, and each of them gets an independent reward sample. All agents together aim to minimize the regret.

\textbf{Heterogeneous settings}. In heterogeneous settings, each agent $m \in \mathcal{M}$ has access only to a subset of arms $\mathcal{K}\upbra{m} \subset \mathcal{K}$.
Assume $\abs{\mathcal{K}\upbra{m}} > 1$ for each agent $m \in \mathcal{M}$ to exclude trivial cases.
For each agent \(m\), we denote the arm with the highest reward mean in $\mathcal{K}\upbra{m}$ as $k_*\upbra{m} = \argmax_{k\in\mathcal{K}\upbra{m}} \mu(k)$, called agent \(m\)'s \emph{local optimal arm}.
Similarly to the homogeneous setting, after agent $m$ selects arm $k_t\upbra{m}$ in round $t$, the environment also reveals a sub-Gaussian reward $X_t\upbra{m,0}(k_t\upbra{m})$ with expectation $\mu(k_t\upbra{m})$. Agents share their information with others. However, due to different local optimal arms \(k_*\upbra{m}\) benchmarks, the regret is defined differently:
\begin{equation}\label{eq:R_Heter}\textstyle
	R(T) = T\sum_{m=1}^M\mu(k_*\upbra{m}) - \sum_{m=1}^M\sum_{t=1}^T\mu(k_t\upbra{m}).
\end{equation}


\textbf{Threat model}.
In both the homogeneous and heterogeneous scenarios, agent $m$ selects an arm $k_t\upbra{m}$ from its respective arm set ($\mathcal{K}$ in homogeneous settings and $\mathcal{K}\upbra{m}$ in heterogeneous settings). The environment generates sub-Gaussian pre-attack reward feedback, denoted as $X_t\upbra{m,0}(k_t\upbra{m})$. We assume that there exists an \emph{attacker} who chooses a subset of agents, $\mathcal{D} \subseteq \mathcal{M}$, as the target to attack. It can observe pre-attack rewards from all agents, and manipulate those from $m \in \mathcal{D}$ into the post-attack reward $X_t\upbra{m}(k_t\upbra{m})$ before returning them to the agent.
It is worth noting that the agents are oblivious to the attacker's presence and rely on this post-attack reward for subsequent decision-making. In homogeneous settings, analogous to attacks in single-agent scenarios, the attacker attempts to force \emph{all} agents to pull a target arm for $T - o\left(T\right)$ times, incurring a cumulative attack cost of only 
\begin{equation*}
    C(T) \coloneqq \sum_{m\in\mathcal{D}}\sum_{t = 1}^T\left|X_t\upbra{m,0}(k_t\upbra{m}) - X_t\upbra{m}(k_t\upbra{m})\right| = o(T).
\end{equation*}
In heterogeneous settings, the attacker knows the local arm set $\mathcal{K}^{(m)}$ for every agent $m\in \mathcal{M}$, and its objective is to maximize the number of agents (affected agents) that suffer linear regret, as achieving the target arm objective with sublinear costs may not be feasible in this context. The detailed reason is discussed in \Cref{sec:hetero}.
\section{Attacks in Homogeneous Settings}\label{sec:homo}
\matob algorithms in homogeneous settings can be broadly classified into two categories~\citep{wang2023achieve}: fully distributed algorithms and leader-follower algorithms. The distinction between these algorithms lies in the presence or absence of a central agent (or server) that determines the actions of agents. In fully distributed algorithms, all agents participate in exploration and exploitation. Conversely, in leader-follower algorithms, the leader primarily manages exploration and plays a pivotal role. 
In this section, we focus on adversarial attacks against fully distributed algorithms due to the space limit. In the appendix, we also show that targeting only the leader is adequate for misleading leader-follower algorithms.


\textbf{Attacks against CO-UCB}.
Despite the abundance of fully distributed algorithms presented in the literature, for consistency with the heterogeneous settings in \cref{sec:hetero}, we choose Cooperative Upper Confidence Bound (CO-UCB)~\citep{yang2022distributed}, a representative \matob algorithm that functions effectively in both homogeneous and heterogeneous environments, as our attack target. In CO-UCB, in each round, each agent pulls the arm with the highest UCB index and shares its reward observation immediately with others. As we assume CO-UCB knows the range of the mean reward as a prior, it constrains all UCB indices to the interval $[0,b]$. Although our attack methodologies and analysis are crafted for CO-UCB, they can also be extended to other distributed algorithms. A detailed example of attacking UCB-TCOM~\citep{wang2023achieve}, the state-of-art homogeneous algorithm with efficient communication design, can be found in the appendix.

We first introduce some notations. Let $\hat{n}_t(k)$ denote the total number of times that arm $k$ is pulled by all $M$ agents globally up to time $t$, and let $\hat{\mu}_t(k)$ represent the \emph{post}-attack empirical mean associated with $\hat{n}_t(k)$. Without loss of generality, we choose the worst arm $K$ as the target arm, as it leads to the highest attack costs. Our goal is to mislead the agents running CO-UCB in order to convince them to pull the target arm $T - o(T)$ times with $o(T)$ attack costs.
In the homogeneous setting, we can achieve this goal by merely attacking a single agent. Intuitively, since agents consistently share their reward observations, the manipulated rewards from one agent are disseminated to the rest, influencing their choices. To this end, we select an arbitrary agent, $m$, to attack. In round $t$, if its chosen arm $k$ is not the target arm $K$, we manipulate its reward $X_t^{(m,0)}(k)$ to fulfill the following inequality:
\begin{equation}\label{equ:IC-inequality}
    \begin{aligned}
        \hat{\mu}_t(k) \le \hat{\mu}_t(K) - 2\beta(\hat{n}_t(K)) - \Delta_0,
    \end{aligned}
\end{equation}
where \[
\begin{split}
    \hat{\mu}_t(k)  = \frac{\hat{\mu}_{t-1}(k)\hat{n}_{t-1}(k) + \sum_{s=1}^MX_t\upbra{s,0}(k)-\gamma^{(m)}(t)}{\hat{n}_{t}(k)},\,
    \beta(N)  \coloneqq  \sqrt{\frac{1}{2N}\log \frac{\pi^2KN^2}{3\delta}},
\end{split}
\]
and $\gamma^{(m)}(t)$ is the attack value,
$\Delta_0 > 0$ and $\delta >0$ are the parameters of the attack strategy. 
This strategy is similar to the attack design against single-agent bandits~\citep{jun2018adversarial}. It guarantees that the empirical means of non-target arms, after the attack, consistently remain below that of the target arm.
We define the reward mean gap of two arms as $\Delta(k,k') \coloneqq \mu(k) - \mu(k')$. \Cref{thm:homo} provides the upper bound of the cumulative cost $C(T) = \sum_{t=1}^{T}|\gamma^{(m)}(t)|$
(note we only attack one agent \(m\)) for CO-UCB with confidence parameter $\alpha$.
\begin{theorem}\label{thm:homo}
    Suppose $T>K, \delta < 1/2, \Delta_0 > 0$. With probability at least $1-\delta$, our attack strategy misguides \emph{all agents}, running the CO-UCB algorithm, to choose the target arm $K$ at least $T - o(T)$ times, or formally, \[
    \hat{n}_T(K) \ge T - \frac{\alpha(K-1)}{2\Delta_0^2}\log T,
    \] 
    using a cumulative cost at most
    \begin{equation*}\textstyle
        \begin{aligned}
            C(T) \le
              \left(\frac{\alpha}{2\Delta_0^2} \log T\right)\sum_{k<K}(\Delta(k,K) + \Delta_0) + \frac{4(K-1)\sigma}{\Delta_0}\sqrt{\alpha\log T\log\frac{K\pi^2\alpha^2(\log T)^2}{12\delta\Delta_0^4}}.
        \end{aligned}
    \end{equation*}
\end{theorem}

The cumulative attack cost is in the order of \(O(K \log T)\), suggesting that even when attacks are limited to a single agent, no additional costs are incurred, as the corrupted observations would propagate to the whole system. Intuitively, it is equivalent to evenly spreading an \(O(K/M \log T)\) cost to each agent.
Notably, the total cost is independent of the number of agents \(M\), which highlights the vulnerabilities in \matob, as the cost does not escalate with an increase in the number of agents.

\section{Attacks in Heterogeneous Settings}\label{sec:hetero}
In this section, we study adversarial attacks on \matob in heterogeneous settings, where the agent can have distinct local arm sets. We assume that all agents run the CO-UCB~\citep{yang2022distributed} algorithm, which is a representative \matob algorithm for heterogeneous settings\footnote{There are other \matob algorithms~\citep{baek2021fair,wang2023explore} that enjoy better regrets on the heterogeneous setting. Since all these algorithms rely on UCB, we pick CO-UCB as an example to devise the attack approach and believe it can be extended to other UCB-based heterogeneous \matob algorithms.}.
We first discuss the viability of different attack objectives. Following that, we propose attack strategies with the appropriate objective and offer theoretical analyses of their associated costs.

\subsection{Attack Objectives}\label{subsec:conflict}
\begin{figure}[t]
	\centering
	\begin{subfigure}[b]{0.3\textwidth}
		\centering
		\includegraphics[width=\textwidth]{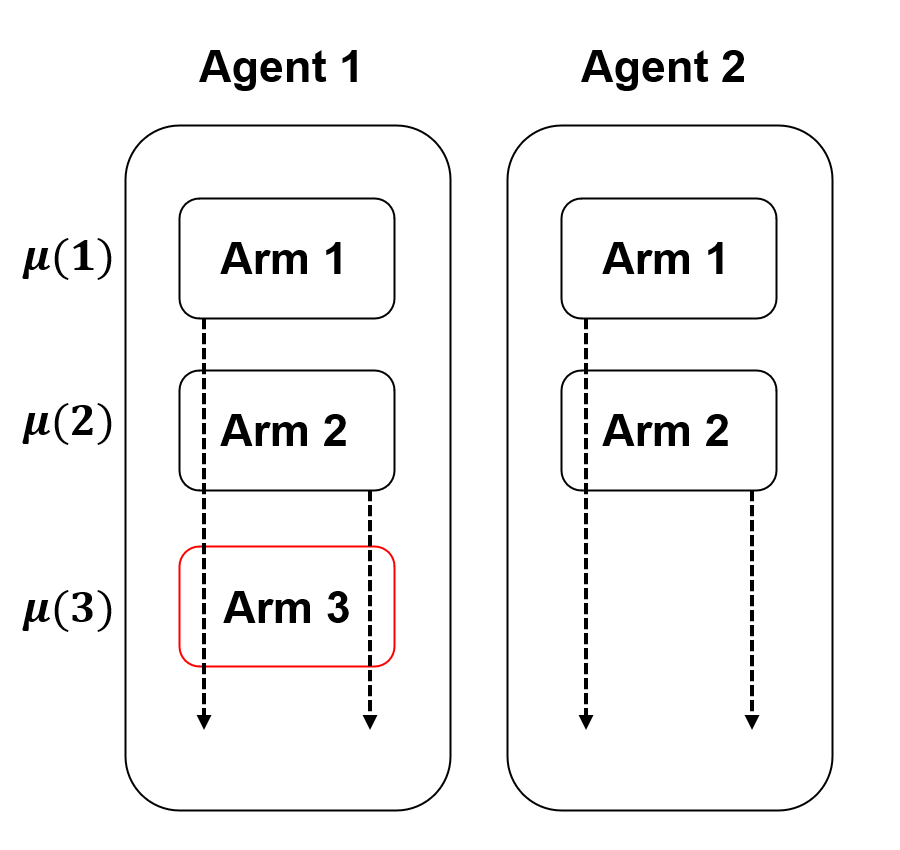}
		\caption{Target arm attack}
		\label{fig:obj_1}
	\end{subfigure}
	\begin{subfigure}[b]{0.3\textwidth}
		\centering
		\includegraphics[width=\textwidth]{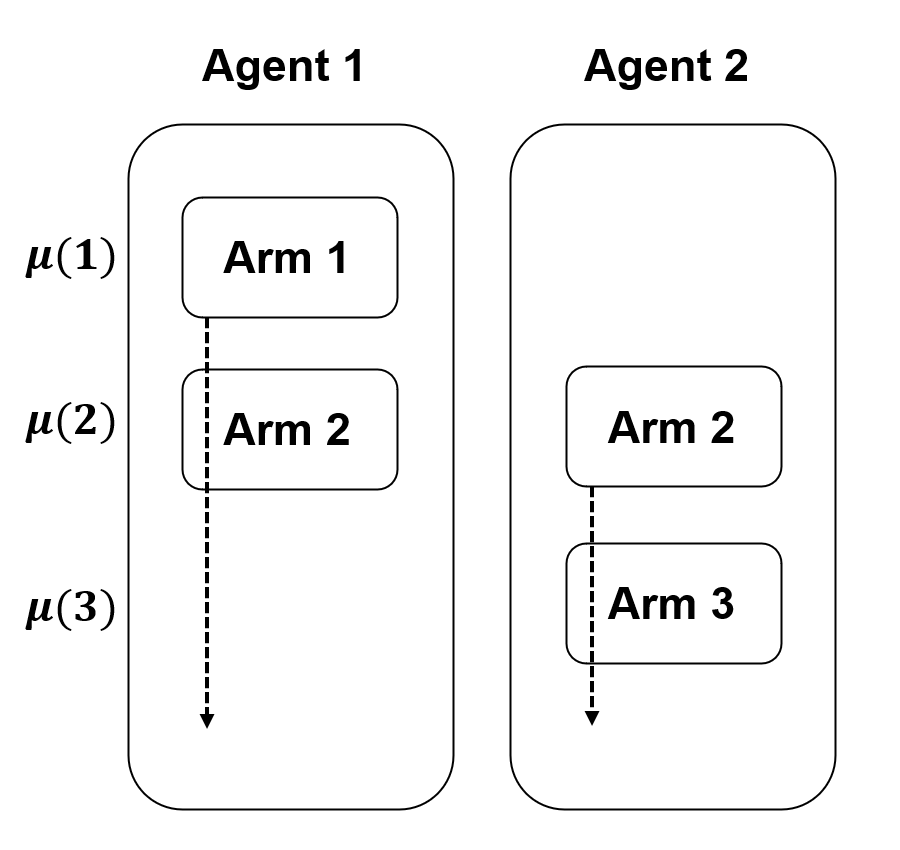}
		\caption{Linear regret attack}
		\label{fig:obj_2}
	\end{subfigure}
	\caption{Illustration of attack objectives}
	\label{fig:obj}
\end{figure}
While the majority of heterogeneous \matob algorithms are derivatives of their homogeneous counterparts, the distinctiveness introduced by agent heterogeneity poses novel challenges in devising adversarial attacks. We first consider the original target arm attack as in the homogeneous settings, which aims to deceive all agents into selecting a target arm $T-o(T)$ times. Intriguingly, in heterogeneous settings, achieving this objective might require linear attack costs. A simple example of two agents is shown in \Cref{fig:obj_1}. We consider the target arm to be arm $3$ in agent $1$, and intuitively, once arm $1$ or $2$ is pulled, the attacker needs to decrease their rewards. However, given this heterogeneous setup, agent $2$ only has access to arms $1$ and $2$. Therefore, it is compelled to select them repeatedly, and their reward samples are subsequently sent to agent $1$. As a result, to deceive agent $1$ into frequently selecting arm $3$, linear attack costs on agent $2$ become necessary.
\Cref{prop:obj_1} formally shows the necessity of linear costs to realize the target arm attack.
\begin{proposition}\label{prop:obj_1}
	For any attack strategy that can successfully mislead the agents running CO-UCB in \Cref{fig:obj_1} to pull the target arm $T-o(T)$ times, its attack cost is at least $C(T) \ge c T$ for some constant $c>0$.
\end{proposition}

Thus, the target arm attack may not be an appropriate attack objective in heterogeneous environments. Shifting our focus, we consider an alternative attack objective: misleading all agents toward linear regrets. While this objective seems less stringent, as it merely mandates agents not to choose their local optimal arms, the intrinsic heterogeneity of available arms brings forth complexities. Notably, there might be agents that, given the disparity in arm sets, cannot be simultaneously misguided towards linear regrets with only sublinear costs. Such agents are termed as \emph{``conflict''} agents. An example of this scenario is depicted in \Cref{fig:obj_2}.
Our objective is to deceive agents $1$ and $2$, preventing them from selecting their locally optimal arms, and thereby incurring linear regrets. Notably, while arm $2$ is suboptimal for agent $1$, it is optimal for agent $2$. After the attacks, agent $1$ should pull arm $2$ linear times, and those rewards will be communicated to agent $2$ to affect its arm selection. Consequently, ensuring agent $2$ chooses its suboptimal arm $3$ almost linear times incurs linear attack costs.
In~\Cref{prop:obj_2}, we further demonstrate that linear costs are inevitable when pursuing this attack objective.
\begin{proposition}\label{prop:obj_2}
	For any attack strategy that can successfully mislead all agents running CO-UCB in \Cref{fig:obj_2} to suffer linear regrets, its attack cost is at least $C(T) \ge c T$ for some constant $c>0$.
\end{proposition}
Although sublinear attack costs might not be sufficient in leading all agents to experience linear regrets, they can still influence a subset of the agents. This realization prompts the final attack objective explored in this section: leveraging sublinear attack costs to misguide the maximum number of agents, aiming to increase the overall count of agents enduring linear regrets. To realize this goal, it is necessary to identify the largest set of agents that do not have conflicts. The condition under which two agents are deemed to be in conflict is given in \Cref{cond:conflict}.
\begin{condition}[Agents conflict]\label{cond:conflict}
	For any two agents $m,m' \in \mathcal{M}$, if $|\mathcal{K}\upbra{m} \setminus \{k\upbra{m}_*, k\upbra{m'}_*\}|=0$ or $|\mathcal{K}\upbra{m'} \setminus \{k\upbra{m}_*, k\upbra{m'}_*\}|=0$, we say agent \(m\) conflicts with agent $m'$, i.e., they cannot be simultaneously attacked with sublinear costs.
\end{condition}
Given this condition, we introduce a greedy algorithm designed to identify the largest group of affected agents without conflicts. Furthermore, we devise attack strategies tailored for both known (\Cref{subsec:oracle-attack}) and unknown (\Cref{subsec:learning-then-attack}) environments.

\subsection{Oracle Attack Strategy}\label{subsec:oracle-attack}
We study a scenario in which the attack algorithm has prior knowledge of the environment. Specifically, the attacker knows the reward ranking of all arms. This premise is less stringent than the ``oracle attack'' assumption from previous studies~\citep{jun2018adversarial,liu2019data}, which requires precise knowledge of the mean rewards for every arm. Note that
while we similarly label our approach as an ``oracle attack'' and use the $\mu(k)$ notation in our algorithms, they only rely on the relative arm ranking but not the exact means.

\textbf{Affected Agents Selection}.
As previously mentioned, the first step for the attacker is to identify the largest group of agents without conflicts as the attack objective. Although \Cref{cond:conflict} allows us to check the conflict status between any pair of agents, finding the maximal conflict-free agent group poses a non-trivial combinatorial optimization challenge. 
To solve this problem, we propose the Affected Agents Construction (\aas) algorithm, described in \Cref{alg:aas}.
Initially, the algorithm categorizes all agents into separate sets based on their local optimal arms, i.e., agent set \(\mathcal{M}_*(k)\coloneqq \{m\in\mathcal{M}: k_*\upbra{m} = k\}\) contains all agents whose local optimal arm is arm \(k\), and sort them by set size \(M_*(k) \coloneqq \abs{\mathcal{M}_*(k)}\). Following this, a greedy set selection process is employed. The group of selected agents is maintained by $\mathcal{D}_0$. The algorithm examines sets from $M_*(\omega(1))$ to $M_*(\omega(k))$ sequentially, where \(M_*(\omega(i))\) is the size of the \(i^{\text{th}}\) largest 
 agent sets among \(\{\mathcal{M}_*(1),\dots,\mathcal{M}_*(K)\}\). If all agents within $\mathcal{M}_*(\omega(k))$ do not conflict with the current group $\mathcal{D}_0$, which is checked by \cref{eq:max-agent-condition}, they are incorporated into $\mathcal{D}_0$, and the subsequent set $\mathcal{M}_*(\omega(k+1))$ will be examined. It eventually outputs the target group $\mathcal{D}_0$ and the local optimal arms' set $\mathcal{K}_0$ of agents in the group.
The following theorem provides the theoretical guarantee of \Cref{alg:aas}.
\begin{theorem}\label{thm:aas}
	\Cref{alg:aas} finds a $(1-1/e)$-approximate solution, i.e., the cardinality of the output group $|\mathcal{D}_0| \ge (1-1/e) D_{\max}$, where $D_{\max}$ is the size of the largest conflict-free agent group.
\end{theorem}
We prove this approximation ratio using the property of submodular functions in the appendix.

\begin{algorithm}[t]
	\caption{\aas: Affected Agents Selection}\label{alg:aas}
	\begin{algorithmic}[1]
            \State \textbf{input} local arm sets \(\mathcal{K}\upbra{m},\,\forall m\in\mathcal{M}\)
		\State \textbf{initialize} \(\mathcal{D}_0\gets \emptyset\), \(\mathcal{K}_0\gets\emptyset\)
		\State Classify all agents into \(\{\mathcal{M}_*(1),\dots,\mathcal{M}_*(K)\}\) according to their local optimal arms
		\State Sort these agent sets according to set size such that \(\mathcal{M}_*(\omega(k))\) is the agent subset that contains the \(k^{\text{th}}\) largest number of agents
		\For{\(k=1,2,\dots, K\)}
		\If{for all agent \(m \in \mathcal{D}_0 \cup \mathcal{M}_*(\omega(k)) \) we have
        \begin{equation}\label{eq:max-agent-condition}
			\abs*{\mathcal{K}\upbra{m}\setminus (\mathcal{K}_0 \cup \{\omega(k)\})} > 0
		\end{equation}}
		\State \(\mathcal{D}_0 \gets \mathcal{D}_0\cup \mathcal{M}_*(\omega(k))\), \(\mathcal{K}_0 \gets \mathcal{K}_0 \cup \{\omega(k)\}\)
		\EndIf
		\EndFor
		\State \textbf{return} $\mathcal{D}_0, \mathcal{K}_0$
	\end{algorithmic}
\end{algorithm}

\textbf{Target Agents Selection}.
With output agent group $\mathcal{D}_0$ as the attack objective, a natural question arises: is it feasible to attack this group using only a limited number of agents, similar to the scenario in homogeneous settings? To address this, we introduce the Target Agents Selection (\tas) algorithm, which refines the selection of target agents to attack within $\mathcal{D}_0$.
As described in \Cref{alg:tas}, for each agent $m \in \mathcal{D}_0$, it finds $k_0^{(m)}$, its local arm with the highest mean reward excluding all arms in $\mathcal{K}_0$ (Lines~\ref{line:exclude}-\ref{line:k_0^m}). Then, for each arm $k \in \mathcal{K}_0$, it checks $\mathcal{D}_{0,*}(k)$, which contains all agents in $\mathcal{D}_0$ with local optimal arm $k$, i.e., $\mathcal{D}_{0,*}(k) = \{m\in \mathcal{D}_0: k_*\upbra{m} = k\}$; it chooses the agent in $\mathcal{D}_{0,*}(k)$ with the lowest $\mu(k_0^{(m)})$ and include it into the target agent set $\mathcal{G}_0$. Intuitively, in order to attack each arm $k\in \mathcal{K}_0$, \Cref{alg:tas} selects the agent in $\mathcal{D}_0$ that is most likely to pull arm $k$ very often, since its alternative action $k_0^{(m)}$ closest to $k$ has the least attractive mean reward. Such target agent selection will help us control the number of times that $k\in \mathcal{K}_0$ is pulled by agents outside $\mathcal{G}_0$, ensuring successful attacks.

\begin{algorithm}[t]
	\caption{\tas: Target Agents Selection}\label{alg:tas}
	\begin{algorithmic}[1]
		\State \textbf{input} \(\mathcal{D}_0, \mathcal{K}_0\)
        \State \textbf{initialize} \(\mathcal{G}_0\gets \emptyset\)
		\For{\(m \in \mathcal{D}_0\)} \Comment{for each affected agent}
        \State $\mathcal{K}^{(m)}_0 \gets \mathcal{K}^{(m)} \setminus \mathcal{K}_0$ \label{line:exclude}
        \State $k_0^{(m)} \gets \argmax_{k \in \mathcal{K}^{(m)}_0} \mu(k)$ \label{line:k_0^m}
		\EndFor
        \For{\(k \in \mathcal{K}_0\)} \Comment{for local optimal arms to attack}
        \State $g(k) \gets \argmin_{m \in \mathcal{D}_{0,*}(k)} \mu(k_0^{(m)})$
        \State $\mathcal{G}_0 \gets \mathcal{G}_0 \cup \{g(k)\}$
  		\EndFor
		\State \textbf{return}  $\mathcal{G}_0$
	\end{algorithmic}
\end{algorithm}

\textbf{Attack Strategy and Analysis}.
\begin{algorithm}[t]
	\caption{Oracle Attack}\label{alg:OA}
	\begin{algorithmic}[1]
		\State \textbf{input} \(\Delta_0\)
        \State $\mathcal{D}_0, \mathcal{K}_0 \gets$ \texttt{\aas}$((\mathcal{K}\upbra{m})_{m\in\mathcal{M}})$
        \State $\mathcal{G}_0 \gets$ \texttt{\tas}$(\mathcal{D}_0, \mathcal{K}_0)$
        \For{\(t = 1,2,\cdots,\)}
        \For{agent $m \in \mathcal{G}_0$}
            \If{$k_t^{(m)} \in \mathcal{K}_0$}
            \State Attack $k_t^{(m)}$ according to \cref{eq:OA_attack}
            \EndIf 
  		\EndFor
        \EndFor
	\end{algorithmic}
\end{algorithm}
We present the Oracle Attack (OA) algorithm as detailed in \Cref{alg:OA}. Initially, it invokes both \aas and \tas to select the target agent set \( \mathcal{G}_0 \) responsible for executing the attacks. Subsequently, whenever an agent \( m \in \mathcal{G}_0 \) chooses an arm \( k \in \mathcal{K}_0 \), it attacks $k$ with attack value \( \gamma_{t}^{(m)}(k) \) to satisfy the ensuing inequality:
\begin{equation}\label{eq:OA_attack}
    \hat{\mu}_{t}(k) \leq \min_{k' \in \mathcal{K}\backslash\mathcal{K}_0}\{\hat{\mu}_{t-1}(k') - 2\beta(\hat{n}_{t-1}(k')) - \Delta_0\},
\end{equation}
where $\hat{\mu}_{t}(k) = \frac{\hat{\mu}_{t-1}(k)\hat{n}_{t-1}(k) + \sum_{m'=1}^MX_t\upbra{m',0}(k) - \gamma_t\upbra{m}(k)}{\hat{n}_t(k)}$. 
We define $\Delta_{\min} = \min_{k} \Delta(k, k+1)$ and provide the attack cost analysis for \Cref{alg:OA}.
\begin{theorem}\label{thm:OA}
	Suppose $T > T_0, \alpha > 2$, where \(T_0\) is a time-independent constant fulfills~\eqref{eq:solve_T_0}. With probability at least $1 - \delta$, \Cref{alg:OA} misguides the agents to suffer regret at least
	\begin{equation*}\textstyle
		\begin{aligned}
			R(T) \ge \Delta_{\min}\sum_{k \in \mathcal{K}_0}  
			\left({M}_*(k)T-\frac{2\alpha\log T}{\Delta_0^2} \right),
		\end{aligned}
	\end{equation*}
	using the cumulative cost at most
	\begin{align*}\textstyle
			C(T)  \le \sum_{k \in \mathcal{K}_0} \left(\frac{\alpha\log T}{2\Delta_0^2}(\Delta(k, K) + \Delta_0) \right. + T_0 
			     \left.~+ \frac{4\sigma}{\Delta_0} \sqrt{\alpha\log T \log\frac{K\pi^2\alpha^2(\log T)^2}{12\delta\Delta_0^4}}\right).
	\end{align*}
$T_0$ is a feasible solution of the following equation
\begin{equation}\label{eq:solve_T_0}
    \frac{t}{\log t} \ge \max_{k \in \mathcal{K}_0} c_{k}, 
\end{equation}
where \(c_{k}= c_{k,1} + c_{k,2} + c_{k,3}\) and 
\begin{align}\nonumber
     c_{k,1} &= \sum_{k' \in \mathcal{K}_{0}^{(g(k))}\setminus \{k_{0}^{(g(k))}\}} \frac{\alpha}{2\Delta^2(k_{0}^{(g(k))}, k')},\nonumber\\
     c_{k,2} &= \frac{\alpha}{2\min_{m\in\mathcal{D}_{0,*}(k): k_0\upbra{m}\neq k_0\upbra{g(k)}}\Delta^2(k_{0}^{(m)}, k_{0}^{(g(k))})},\nonumber\\
     c_{k,3} &= |\mathcal{K}^{(g(k))} \cap \mathcal{K}_0| \cdot \frac{\alpha}{\Delta_0^2}.\nonumber
\end{align}
\end{theorem}
The attack cost follows the order of \(O(|\mathcal{K}_0| \log T)\) and is independent of the number of affected agents, \(|\mathcal{D}_0|\). This implies that a small attack cost can have a substantial impact on numerous agents in heterogeneous settings. Notice that directly comparing this result with \Cref{thm:homo} would be unfair, given the distinct objectives and settings they address. As the proof of this theorem is one of our main technical contributions, we briefly discuss its key idea below.

\textit{Proof Challenge and Key Idea}. 
The main challenge in this proof arises from our choice of target agents. These agents can lower the post-attack empirical means of arms in $\mathcal{K}_0$ only when they pull these arms. However, non-target agents may also pull these arms, yielding non-attacked samples that increase the empirical means towards the true means. This issue is especially prominent in heterogeneous settings where target and non-target agents, due to their distinct arm sets, might choose different arms to pull. 
Conversely, in homogeneous settings, 
all agents have access to all arms, consistently offering opportunities for attacks. To address this, we use \tas to choose up to $|\mathcal{K}_0|$ agents from $\mathcal{D}_0$. These target agents are the most likely ones to frequently pull arms in $\mathcal{K}_0$ because their top arms (after excluding those in $\mathcal{K}_0$), denoted as $k_0^{(m)}$, possess the least attractive mean rewards.

Formally, we prove that for any $t > T_0, k\in \mathcal{K}_0$, our target agents in $\mathcal{G}_0$ will prioritize pulling arm $k$ before non-target agents, assuring enough attack opportunities. The value of $T_0$ is calculated by \cref{eq:solve_T_0}. To elucidate, for each arm $k\in \mathcal{K}_0$, target agent $g(k)$ is mainly responsible for attacking it. If $g(k)$'s best arm excluding those in $\mathcal{K}_0$, $k_0^{(g(k))}$, has been sufficiently pulled, i.e., $\hat{n}_t(k_0^{(g(k))}) \ge c_{k,2} \log t$, it will consistently be the first to pull arm $k$ in $\mathcal{D}_{0,*}(k)$. We then bound the number of times that $g(k)$ does not pull $k_0^{(g(k))}$: it can pull local arms worse than $k_0^{(g(k))}$ at most $c_{k,1} \log t$ times, and local arms within $\mathcal{K}_0$ up to $c_{k,3} \log t$ times. Hence, $\hat{n}_t(k_0^{(g(k))}) \ge t - c_{k,1} \log t - c_{k,3} \log t$, and we want $t - c_{k,1} \log t - c_{k,3} \log t \ge c_{k,2} \log t$. We need to ensure this equation for every $k$, leading to \cref{eq:solve_T_0}. Since the right-hand side of \cref{eq:solve_T_0} is a problem-dependent coefficient, we can derive $T_0$ which is independent of $t$. For subsequent rounds $t > T_0$, it is easy to prove the successful attacks with bounded attack costs.

\subsection{LTA: Learning-Then-Attack Strategy}\label{subsec:learning-then-attack}
In this section, we further relax the oracle assumption that the reward ranking of all arms is unknown. This creates a significant challenge as \Cref{alg:aas} and \Cref{alg:tas} need knowledge of the arm ranking to determine the largest set of affected agents and select target agents to attack. To address this issue and adapt the attack design from \Cref{alg:OA} to this scenario, we introduce the Learning-Then-Attack (LTA) algorithm, detailed in \Cref{alg:LTA}. LTA operates in two phases: an initial learning stage, where it discerns reward means and arm rankings, followed by an attack phase akin to \Cref{alg:OA}.
\begin{algorithm}[t]
	\caption{LTA: Learning-Then-Attack}\label{alg:LTA}
	\begin{algorithmic}[1]
		\State Input: confidence parameter $\delta$, minimal mean difference $ \Delta_{\min}$, threshold \(L\)
		\LeftComment{Stage 1: Attack to learn full rank}
		\While{$\min_{k\in \mathcal{K}} \hat{n}_t(k) < L$}
		\ForAll{agent \(m\in\mathcal{M}\)}
		\State Observe the pulled arm \(k_t\upbra{m}\)
		\If{\(\hat{n}_t(k_t\upbra{m}) < L\)}\label{line:less-than-threshold}
            \State Attack $k_t\upbra{m}$ according to~\eqref{eq:incentive_attack} \label{line:incentive-attack}
		\EndIf
		\If{\(\hat{n}_t(k_t\upbra{m}) = L\)}\label{line:reach-threshold}
		\State Recover $k_t\upbra{m}$ to unbiased mean \label{line:recover}
		\EndIf
		\EndFor
		\EndWhile
		\LeftComment{Stage 2: Attack to mislead agents}
            \State Run oracle attack in Algorithm~\ref{alg:OA}
	\end{algorithmic}
\end{algorithm}

\textbf{Attacker's Learning Stage}.
We now delve into the rank learning stage.
This ranking is pivotal for executing \aas (\cref{alg:aas}) and \tas (\cref{alg:tas}), as they necessitate the knowledge of each agent's local optimal arms to optimize the affected agent group and minimize the target agent subset.
To acquire this ranking, the attacker needs to compel agents to pull each arm multiple times. This ensures a sufficient number of pre-attack samples, allowing for a clear distinction between the Lower Confidence Bounds (LCBs) and Upper Confidence Bounds (UCBs) for every arm pair. 

We initialize the UCB of each arm with a relatively high value. Given that agents choose arms based on UCB algorithms, the attacker needs to stimulate agents to collect ample samples for suboptimal arms by increasing their UCB values through attacks.
More specifically, if the attacker wants to accumulate samples of arm \(k\), it needs to attack the arm's reward such that the subsequent condition is met (Line~\ref{line:incentive-attack}),
\begin{equation}\label{eq:incentive_attack}
	\UCB_t(k) > \UCB_t(k'),\,\forall k\neq k'.
\end{equation}
Once the number of times of pulling the arm \(k\) reaches a threshold \(L \coloneqq \ceil*{\frac{2\log (2K/\delta)}{\Delta^2_{\min}}}\) (Line~\ref{line:reach-threshold}), the attacker resets the arm's empirical mean (remove the prior attacks on the arm) (Line~\ref{line:recover}). This ensures that arms that have not been sufficiently sampled in the past will be selected later.


\textbf{Reduce Number of Target Agents}
\begin{condition}[Arm accessibility]\label{cond:cM}
    For each arm \(k\in\mathcal{K}\), there are at least \(cM\in \mathbb{N}^+\) agents in the target agent set \(\mathcal{S}_0\subseteq \mathcal{M}\) being able to access it, where \(c>0\) is the arm accessible rate among target agents.
    Formally, the condition is \[
    \left|\{m\in \mathcal{S}_0: k\in\mathcal{K}\upbra{m} \}\right| \ge cM,\quad \forall k \in\mathcal{K}.
    \]
\end{condition}
We note that this target agent set \(\mathcal{S}_0\) for learning the arm ranking can be different from \(\mathcal{G}_0\) chosen in~\Cref{alg:tas},
and this condition is not restrictive. For example, letting \(\mathcal{S}_0\) be a subset of agents whose local optimal arms together cover the full arm set \(\mathcal{S}\), choosing \(c = \frac{1}{M}\)
is always a valid choice.

Condition~\ref{cond:cM} ensures that during each round of the learning stage, there are a minimum of \(cM\) \emph{effective} observations. Here, ``effective observations'' denote the observations of arms with sample sizes below the threshold \(L\). During this stage, agents are motivated to select arms (provided they are in their local arm sets) with observations fewer than \(L\). Whenever there is an arm that hasn't reached this threshold, a minimum of \(cM\) agents will be motivated to select it, especially when multiple arms are under-observed. Consequently, the learning stage ends after no more than \(\frac{KL}{cM}\) rounds.

\textbf{Analysis}.
Upon completing the learning stage, the attacker has accurate estimates of the reward means for all arms to determine the arm ranking.
It can then apply the oracle attack in~\Cref{alg:OA}.
Notice that in the initial phase of the second stage, when the first time that any arm in $\mathcal{K}_0$ is pulled by a target agent in $\mathcal{G}_0$, the attacker incurs a significant attack cost. This is because the attacker must pay additional costs to alter the unbiased empirical means derived from the learning stage. However, such extra costs can be upper bounded by \(\frac{KL(\Delta(1,K)+\beta(1)+b)}{c}\), where $b$ is the upper bound of the mean rewards.
Subsequent to this initial adjustment, the OA algorithm operates identically to its behavior in the oracle setting.
\begin{theorem}\label{thm:LTA}
	Suppose $T > T_0, \alpha > 2, \delta < 0.5$, where \(T_0\) is a time-independent constant fulfills~\eqref{eq:solve_T_0}. With probability at least $1 - 2\delta$, \Cref{alg:LTA} misguides the agents to suffer regret at least
	\begin{equation*}\textstyle
		\begin{aligned}
			R(T) \ge \Delta_{\min}\sum_{k \in \mathcal{K}_0}  
			\left({M}_*(k)T-\frac{2\alpha\log T}{\Delta_0^2} \right),
		\end{aligned}
	\end{equation*}
	using the cumulative cost at most
	\begin{equation*}\textstyle
		\begin{aligned}
			C(T) & \le \frac{4K\log T}{c\Delta_{\min}^2}(\Delta(1,K)+\beta(1)+b) + \sum_{k \in \mathcal{K}_0} \left(\frac{\alpha\log T}{2\Delta_0^2}(\Delta(k, K) + \Delta_0) \right. + T_0 \\
			     &\left.~+ \frac{4\sigma}{\Delta_0} \sqrt{\alpha\log T \log\frac{K\pi^2\alpha^2(\log T)^2}{12\delta\Delta_0^4}}\right).
		\end{aligned}
	\end{equation*}
\end{theorem}
Compared with the result of the oracle attack, the first term in the attack cost arises from the attacks during the rank learning stage, while the second term is the same as that in \Cref{thm:OA}. If there exists a constant $c$ such that \Cref{cond:cM} holds, the first term will still be in the order of $O(K \log T)$ and matches the cost caused by the oracle attack.
\section{Experiments}\label{sec:exp}
\begin{figure}[t]
    \centering
    \begin{subfigure}[b]{0.3\textwidth}
        \centering
        \includegraphics[width=\textwidth]{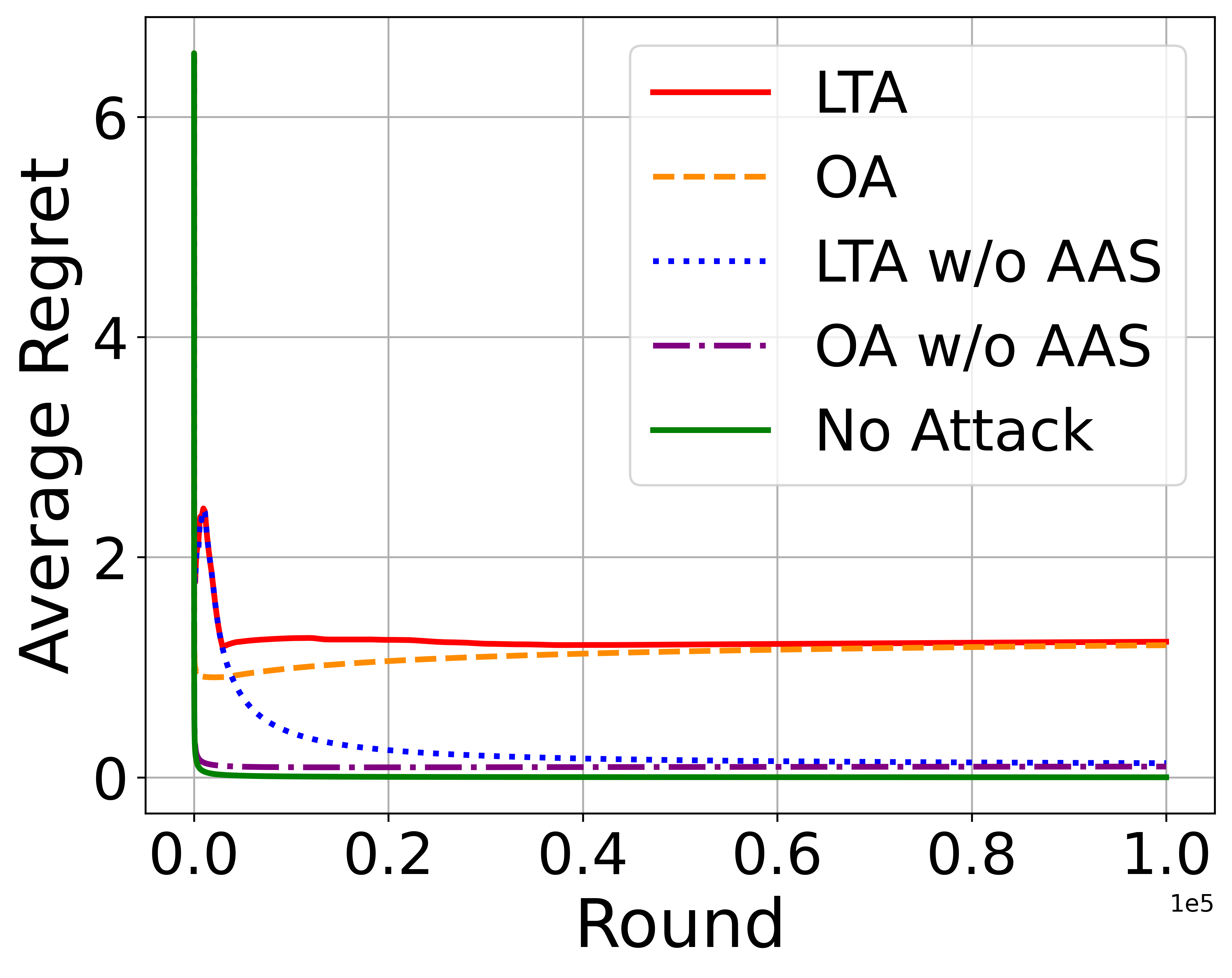}
        \caption{Average Regrets}
        \label{fig:COUCB-Regret}
    \end{subfigure}
    \begin{subfigure}[b]{0.29\textwidth}
        \centering
        \includegraphics[width=\textwidth]{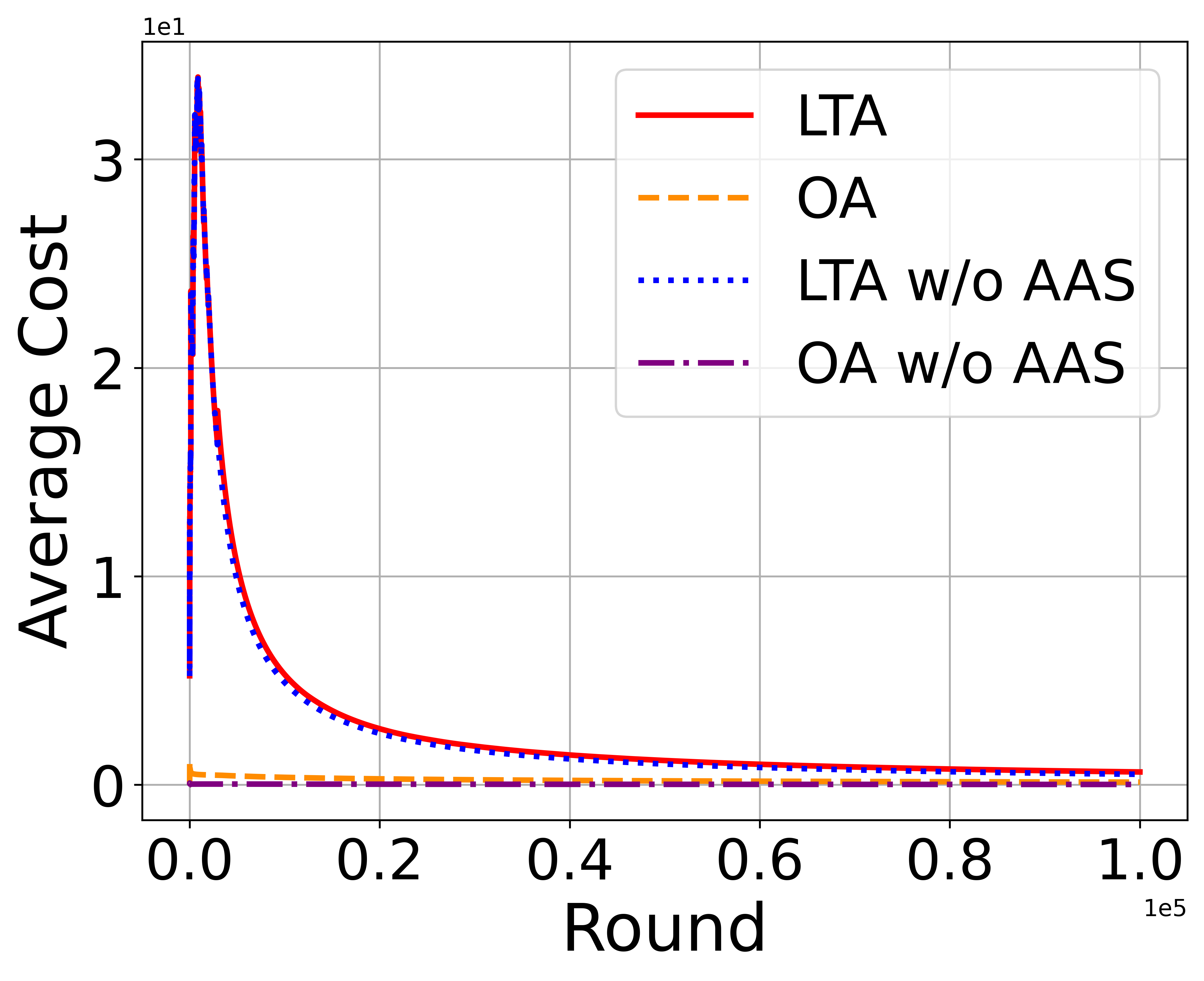}
        \caption{Average Costs}
        \label{fig:COUCB-Cost}
    \end{subfigure}
    \caption{Attacks against CO-UCB.}
    \label{fig:CO-UCB}
\end{figure}
We conduct experiments in both homogeneous and heterogeneous settings. Due to space limitations, we only present the results of heterogeneous settings here. We take $T = 100,000, K = 20, M = 20$. The mean rewards of the arms are randomly sampled within $\left(0, 5\right)$ while ensuring $\Delta_{\min}^2 \ge 0.01$, and the reward of each arm $k$ follows the Gaussian distribution $\mathcal{N}\left(\mu\left(k\right), \sigma\right)$ with $\sigma = 0.1$. Furthermore, each agent $m$ has a set of arms with $|\mathcal{K}\upbra{m}| = 5$. The CO-UCB algorithm takes $\alpha = 10$, and the attack parameters are set to $\Delta_0 = 0.05$ and $\delta = 0.1$. We conducted experiments with five algorithms for comparison: Oracle Attack (OA) with and without Affected Agents Selection (AAS), Leaning-Then-Attack (LTA) with and without AAS, and No Attack. Each experiment was repeated $10$ times.

\Cref{fig:COUCB-Regret} shows the average regret across various algorithms. Both LTA and OA result in the most significant average regrets, primarily due to AAS's capability to identify the most extensive group of affected agents. Without AAS, their average regrets converge to reduced constant values, indicating linear regrets for a limited subset of affected agents.
Figure \ref{fig:COUCB-Cost} shows the average attack costs of different algorithms. All of them approach zero, indicating sublinear cumulative attack costs. In particular, LTA initially incurs higher costs compared to OA, a consequence of the high attack costs during their learning stage.

\section{Concluding Remarks}
This paper pioneers the study of adversarial attacks on \matob. We introduce attack strategies that can effectively mislead \matob algorithms in both homogeneous and heterogeneous settings. One limitation of our work is that there is no known lower bound of the attack cost even for the single-agent bandit setting in the literature, so it is unclear whether our attack strategies are order-optimal. This work also opens up multiple future directions, including the development of attack strategies tailored for competitive multi-agent bandits featuring collisions. Furthermore, our insights into the vulnerabilities of current \matob algorithms set the stage for crafting algorithms resistant to adversarial attacks.

\bibliographystyle{unsrt}
\bibliography{0_reference}

\clearpage
\appendix
\onecolumn
\section*{Appendix}

\section{Proofs}
\subsection{Proof of \texorpdfstring{\Cref{thm:homo}}{T1}}\label{homo-proof}
\begin{proof}
We first define the ``good event'' $E = \{\forall k, \forall t > K: |\hat{\mu}_t\upbra{0}(k) - \mu(k)| < \beta(\hat{n}_t(k))\}$, where $\hat{\mu}_t^0(k)$ is the pre-attack empirical mean of arm $k$ for all agents up to time slot $t$. By Hoeffding's inequality, we can prove that for any $\delta \in (0,1)$, $P(E) > 1- \delta$. To prove \cref{thm:homo}, we introduce two lemmas.

\begin{lemma}\label{lemma:homo-time}
Assume event $E$ holds and $\delta \le 1/2$. For any $k \neq K$ and any $t \ge K$, we have
\begin{equation}
    \hat{n}_t(k) \le \min\{\hat{n}_t(K), \frac{\alpha \log t}{2\Delta_0^2} \}
\end{equation}
\end{lemma}

\begin{proof}
Fix some $t > K$, we assume that $k_t\upbra{m} = k \neq K$. Note that when $k_t\upbra{m} \neq k$, $\hat{n}_t(k)$ will never increase. Also, we assume the last time arm $k$ is selected before time slot $t$ is $t'$. By our attack design in \cref{equ:IC-inequality} , we have:
\begin{equation}\label{eq:homo-time-1}
    \hat{\mu}_{t'}(k) \le \hat{\mu}_{t'}(K) - 2\beta(\hat{n}_{t'}(K)) -\Delta_0.
\end{equation}
On the other hand, arm $k$ is selected in round $t$ means that it has a higher UCB than arm $K$ in round $t-1$:
\begin{equation}\label{eq:homo-time-2}
    \hat{\mu}_{t-1}(k) + \sqrt{\frac{\alpha\log (t-1))}{2\hat{n}_{t-1}(k)}} \ge \hat{\mu}_{t-1}(K) + \sqrt{\frac{\alpha\log (t-1)}{2\hat{n}_{t-1}(K)}}.
\end{equation}
Note that $\hat{\mu}_{t'}(k) = \hat{\mu}_{t-1}(k)$. Substituting \cref{eq:homo-time-1} into \cref{eq:homo-time-2}, we have:
\begin{equation}\label{eq:homo-time-3}
    \begin{aligned}
        \sqrt{\frac{\alpha\log (t-1))}{2\hat{n}_{t-1}(k)}} - \sqrt{\frac{\alpha\log (t-1)}{2\hat{n}_{t-1}(K)}} & \ge \hat{\mu}_{t-1}(K) - \hat{\mu}_{t-1}(k) \\
        & \ge \hat{\mu}_{t-1}(K) - (\hat{\mu}_{t'}(K) - 2\beta(\hat{n}_{t'}(K)) -\Delta_0)  \\
        & \ge \Delta_0 \\
        & > 0,
    \end{aligned}
\end{equation}
where the third inequality is due to event $E$ and the monotonically decreasing property of $\beta$. Therefore, by \cref{eq:homo-time-3} and $\sqrt{\frac{\alpha\log (t-1)}{2\hat{n}_{t-1}(K)}} >0$, the proof is done.
\end{proof}

\begin{lemma}\label{lemma:homo-cost}
Assume event $E$ holds and $\delta \le 1/2$. Denote $\tau(t,k)$ as the rounds in which arm $k$ is selected by all agents up to round $t$. For any $i \neq K$ and any $t \ge K$, we have
\begin{equation}
    \sum_{s \in \tau(t,k)} \gamma\upbra{m}(s) \le \hat{n}_{t}(k)(\Delta(k, K) + \Delta_0 + 3 \beta(\hat{n}_t(K)) + \beta(\hat{n}_t(k)))
\end{equation}
\end{lemma}

\begin{proof}
We assume that $k_{t}\upbra{m} = k \neq K$ for some fixed $t > K$. By \cref{equ:IC-inequality}, we can compute the attack value in each round $t$:
\begin{equation}\label{eq:homo-cost-1}
    \begin{aligned}
        \gamma\upbra{m}(t) & = \left(\hat{\mu}_{t-1}(k)\hat{n}_{t-1}(k)+\sum_{m'=1}^MX_t\upbra{m',0}(k)-\hat{n}_{t}(k)(\hat{\mu}_{t}(K)-2\beta(\hat{n}_{t}(K))-\Delta_0)\right)_+ \\
        & = \left(\hat{\mu}_{t-1}\upbra{0}(k)\hat{n}_{t-1}(k)+\sum_{m'=1}^MX_t\upbra{m',0}(k) - \sum_{s\in \tau(t-1,k)}\gamma\upbra{m}(s) - \hat{n}_{t}(k)(\hat{\mu}_{t}(K)-2\beta(\hat{n}_{t}(K))-\Delta_0)\right)_+.
    \end{aligned}
\end{equation}

If $\gamma^{(m)}(t) = 0$, we change to examine the last time it was greater than zero. We have
\begin{equation}\label{eq:homo-cost-2}
    \begin{aligned}
        \sum_{s\in \tau(t,k)}\gamma\upbra{m}(s) & = \hat{\mu}_{t-1}\upbra{0}(k)\hat{n}_{t-1}(k)+\sum_{m'=1}^MX_t\upbra{m',0}(k) - \hat{n}_{t}(k)(\hat{\mu}_{t}(K)-2\beta(\hat{n}_{t}(K))-\Delta_0) \\
        & = \hat{n}_{t}(k)(\hat{\mu}_{t}\upbra{0}(k) + \hat{\mu}_{t}(K)+2\beta(\hat{n}_{t}(K))+\Delta_0) \\
        & \le \hat{n}_{t}(k)(\Delta(k, K) + \Delta_0 + 3 \beta(\hat{n}_t(K)) + \beta(\hat{n}_t(k))),
    \end{aligned}
\end{equation}
where the last inequality is due to event $E$.
\end{proof}
With \cref{lemma:homo-time}, we can easily get that the target arm $K$ is selected for at least $MT - (K-1)(\frac{\alpha}{2\Delta_0^2}\log T)$ times. As for the cumulative cost, we sum over all non-target arms using \cref{lemma:homo-cost}. We have:
\begin{equation}\label{eq:homo-cost}
    \begin{aligned}
        \sum_{t}^{T} \gamma\upbra{m}(t) & \le \sum_{k = 1}^{K-1}\hat{n}_t(k)(\Delta(k,K)+\Delta_0) + 4\sum_{k = 1}^{K-1} \hat{n}_t(k) \beta(\hat{n}_t(k))\\
        & \le \left(\frac{\alpha}{2\Delta_0^2} \log T\right)\sum_{k<K}(\Delta(k,K) + \Delta_0)   + \frac{4(K-1)\sigma}{\Delta_0}\sqrt{\log T\log\frac{K\pi^2\alpha^2(\log T)^2}{12\delta\Delta_0^4}},
    \end{aligned}
\end{equation}
where in the last inequality, we substitute the chosen times of each arm $k$, $\hat{n}_t(k)$, into $\beta$.
\end{proof}

\subsection{Proof of \texorpdfstring{\Cref{prop:obj_1}}{T1}}\label{sec:obj1:proof}
The proofs of \Cref{prop:obj_1} and \Cref{prop:obj_2} are inspired by the proof of Theorem 2 in \cite{zuo2020near}.

\begin{proof}
Assume arm $3$ is our target arm. We first introduce some notations. Let $\gamma(t,m,k) = X_t\upbra{m,0}(k)-X_t\upbra{m}(k)$. Note that if $k_t\upbra{m} \neq k$, then $\gamma(t,m,k) = 0$. In addition, let $\Gamma(t,k) = \sum_{s=1}^t\sum_{m=1}^2|\gamma(s,m,k)|$.

Assume event $E$ holds. Consider the last time that arm $3$ is pulled (by agent $1$) is round $t + 1$. As the agent selects the arm with the highest UCB in time slot $t$, we have:
\begin{equation}\label{eq:obj1-1}
    \begin{aligned}
        \hat{\mu}_t(3) + \sqrt{\frac{\alpha\log t}{2\hat{n}_t(3)}} \ge \hat{\mu}_t(k) + \sqrt{\frac{\alpha\log t}{2\hat{n}_t(k)}},
    \end{aligned}
\end{equation}
for $k = 1, 2$. Also, we have the following two inequalities:
\begin{equation}\label{eq:obj1-2}
    \begin{aligned}
        \frac{\hat{\mu}_t(3)\hat{n}_t(3) - \Gamma(t,3)}{\hat{n}_t(3)} \le \mu(3) + \beta(\hat{n}_t(3)),
    \end{aligned}
\end{equation}
\begin{equation}\label{eq:obj1-3}
    \begin{aligned}
        \frac{\hat{\mu}_t(k)\hat{n}_t(k) + \Gamma(t,k)}{\hat{n}_t(k)} \ge \mu(k) - \beta(\hat{n}_t(k)),
    \end{aligned}
\end{equation}
for $k = 1, 2$. This is because we consider the absolute values of both positive and negative instances of $\gamma$ up to round $t$ and also utilize event $E$ to upper and lower bound the empirical means. Note that arm $3$ can only be selected by agent $1$. Therefore, we have $\hat{n}_t(3) \le t$, and $\max\{\hat{n}_t(1),\hat{n}_t(2)\} \ge t/2$. Let $i = \argmax_{k=1,2}\{\hat{n}_t(k)\}$. Then, we can construct an instance that needs linear cumulative cost to pull the target arm for linear times.

Suppose $\Delta(2,3) \ge 4, \sigma = 0.5$. As $\hat{n}_T(3) = T-o(T)$, we can assume that there exists some constant $\frac{\sqrt{\delta}}{\pi} < c < 1$, which satisfies $\min\{\hat{n}_t(i), \hat{n}_t(3)\} \ge ct$. Putting \cref{eq:obj1-1,eq:obj1-2,eq:obj1-3} together, we have:
\begin{equation}\label{eq:obj1-4}
    \begin{aligned}
        \frac{\Gamma(t,3)}{\hat{n}_t(3)} + \frac{\Gamma(t,i)}{\hat{n}_t(i)} & \ge \Delta(i,3) - \sqrt{\frac{\alpha\log t}{2\hat{n}_t(3)}} - \beta(\hat{n}_t(3)) - \beta(\hat{n}_t(i)) \\
        & \ge \Delta(2,3) - 3\beta(ct) \\
        & > \Delta(2,3) - 3\sqrt{\frac{\delta}{\pi t}\log\frac{\pi t}{\delta}} \\
        & > 1,
    \end{aligned}
\end{equation}
where the second inequality holds because $\Delta(i,3) \ge \Delta(2,3)$, and we have $\sqrt{\frac{\alpha\log t}{2\hat{n}_t(3)}} \le \beta(\min\{\hat{n}_t(i), \hat{n}_t(3)\}) \le \beta(ct)$ when $\frac{\sqrt{\delta}}{\pi} < c < 1$, as $\beta$ is monotonically decreasing.  The third inequality can be derived from the following result:
\begin{equation*}\textstyle
        \begin{aligned}
            \beta(ct) & =
              \sqrt{\frac{1}{2ct}\log\frac{\pi^2c^2t^2}{\delta}} \\
             & \overset{c<1<\frac{1}{\sqrt{\delta}}}{\le} \sqrt{\frac{1}{2ct}\log(\frac{\pi t}{\delta})^2} \\
             & \overset{c>\frac{\sqrt{\delta}}{\pi}>\frac{\delta}{\pi}}{\le} \sqrt{\frac{\delta}{\pi t}\log\frac{\pi t}{\delta}},
        \end{aligned}
    \end{equation*}
and the last inequality in \cref{eq:obj1-4} is due to $\sqrt{\frac{\log x}{x}} < 1$ for any $x\ge 1$. Therefore, the cumulative cost is
\begin{equation}\label{eq:obj1-5}
    \begin{aligned}
        C(T) \ge \Gamma(T,3) + \Gamma(T,i) \ge cT.
    \end{aligned}
\end{equation}
\end{proof}

\subsection{Proof of \texorpdfstring{\Cref{prop:obj_2}}{T1}}
\begin{proof}
We use the same notations in \cref{sec:obj1:proof} and assume event $E$ holds. In this scenario, we assume that both arms $2$ and $3$ should be selected for linear times; otherwise, one of these agents will not suffer linear regret.

Consider the last time arm $3$ is pulled (by agent $2$) is $t+1$. Then we have the UCB order in time slot $t$:
\begin{equation}\label{eq:obj2-1}
    \begin{aligned}
        \hat{\mu}_t(3) + \sqrt{\frac{\alpha\log t}{2\hat{n}_t(3)}} \ge \hat{\mu}_t(2) + \sqrt{\frac{\alpha\log t}{2\hat{n}_t(2)}}.
    \end{aligned}
\end{equation}
Also, we have the following two inequalities for the same reason of \cref{eq:obj1-2,eq:obj1-3}:
\begin{equation}\label{eq:obj2-2}
    \begin{aligned}
        \frac{\hat{\mu}_t(3)\hat{n}_t(3) - \Gamma(t,3)}{\hat{n}_t(3)} \le \mu(3) + \beta(\hat{n}_t(3)),
    \end{aligned}
\end{equation}
\begin{equation}\label{eq:obj2-3}
    \begin{aligned}
        \frac{\hat{\mu}_t(2)\hat{n}_t(2) + \Gamma(t,2)}{\hat{n}_t(2)} \ge \mu(2) - \beta(\hat{n}_t(2)).
    \end{aligned}
\end{equation}

Suppose $\Delta(2,3) \ge 4, \sigma = 0.5$. As $\hat{n}_T(k) = T-o(T), k=2,3$, we can assume that there exists some constant $\frac{\sqrt{\delta}}{\pi} < c < 1$, which satisfies $\min\{\hat{n}_t(2), \hat{n}_t(3)\} \ge ct$. Put \cref{eq:obj2-1,eq:obj2-2,eq:obj2-3} together, we have:
\begin{equation}\label{eq:obj2-4}
    \begin{aligned}
        \frac{\Gamma(t,3)}{\hat{n}_t(3)} + \frac{\Gamma(t,2)}{\hat{n}_t(2)} & \ge \Delta(2,3) - \sqrt{\frac{\alpha\log t}{2\hat{n}_t(3)}} - \beta(\hat{n}_t(3)) - \beta(\hat{n}_t(2)) \\
        & \ge \Delta(2,3) - 3\beta(ct) \\
        & > \Delta(2,3) - 3\sqrt{\frac{\delta}{\pi t}\log\frac{\pi t}{\delta}} \\
        & > 1,
    \end{aligned}
\end{equation}
with the reason similar to that for \cref{eq:obj1-4}. Therefore, the cumulative cost is
\begin{equation}\label{eq:obj2-5}
    \begin{aligned}
        C(T) \ge \Gamma(T,3) + \Gamma(T,2) \ge cT.
    \end{aligned}
\end{equation}
\end{proof}

\subsection{Proof of \texorpdfstring{\Cref{thm:aas}}{T1}}\label{aas-proof}
\begin{proof} We define a set function $f : 2^{\Omega} \mapsto \mathbb{R}$, which takes the power set of $\Omega := \{\mathcal{M}_*(1), \cdots, \mathcal{M}_*(K)\}$ as input and outputs the largest number of conflict-free agents. We then check its submodularity. For every $X,Y \subseteq \Omega$ with $X \subseteq Y $ and every $x \in \Omega \setminus Y$, we have
\begin{equation}
    f(X \cup \{x\}) - f(X) \ge f(Y \cup \{x\}) - f(Y),
\end{equation}
since if agent $m \in x$ conflicts with any agent $m' \in X$, it must conflict with $m' \in Y$ as well. As a result, $f$ is submodular and the greedy algorithm in \Cref{alg:aas} gives a $(1-1/e)$-approximate solution.

\end{proof}

\subsection{Proof of \texorpdfstring{\Cref{thm:OA}}{T1}}\label{oa-proof}
\begin{proof}
For each arm $k\in \mathcal{K}_0$, we consider the target agent $g(k)$, which is mainly responsible for attacking $k$. We first prove that if $k_0^{(g(k))}$ has been sufficiently pulled, i.e., $\hat{n}_t(k_0^{(g(k))}) \ge c_{k,2} \log t$, $g(k)$ will consistently be the first to pull arm $k$ in $\mathcal{D}_{0,*}(k)$. With $\hat{n}_t(k_0^{(g(k))}) \ge c_{k,2} \log t$, for every $m \in\mathcal{D}_{0,*}(k)$ such that $k_0\upbra{m}\neq k_0\upbra{g(k)}$, we have
\begin{equation}
        UCB_t(k_0^{(g(k))})  = \hat{\mu}_t(k_0^{(g(k))}) + \sqrt{\frac{\alpha \log t}{2\hat{n}_t(k_0^{(g(k))})}} \le \mu(k_0^{(g(k))}) + \Delta(k_0\upbra{m}, k_0\upbra{g(k)}) = \mu(k_0\upbra{m}) \le UCB_t(m).
\end{equation}
Since $k_0^{(g(k))}, k_0\upbra{m}$ are the local optimal arms excluding all arms in $\mathcal{K}_0$ and $UCB_t(k_0^{(g(k))})$ is always less or equal to $ UCB_t(k_0\upbra{m})$, target agent $g(k)$ will pull arm $k$ earlier than any agent $m$, assuring enough attack opportunities.

Next, we want to derive a lower bound of $\hat{n}_t(k_0^{(g(k))})$ for $g(k)$. There are two cases in which $g(k)$ does not pull $k_0^{(g(k))}$: it can pull arms either in $\mathcal{K}_{0}^{(g(k))}\setminus \{k_{0}^{(g(k))}\}$ or in $\mathcal{K}^{(g(k))} \cap \mathcal{K}_0$. We first consider the former case. Since $k_{0}^{(g(k))}$ is the optimal arm in $\mathcal{K}_{0}^{(g(k))}$, the number of pulls of these suboptimal arms can be bounded by
\begin{equation}
    \sum_{k' \in \mathcal{K}_{0}^{(g(k))}\setminus \{k_{0}^{(g(k))}\}} n^{(g(k))}_t(k')\le \sum_{k' \in \mathcal{K}_{0}^{(g(k))}\setminus \{k_{0}^{(g(k))}\}} \frac{\alpha \log t}{2\Delta^2(k_{0}^{(g(k))}, k')}
    = c_{k,1} \log t,
\end{equation}
where $n^{(g(k))}_t(k')$ is the number of times that agent $g(k)$ pulls arm $k'$. The inequality comes from the upper bound of the suboptimal arm pulls for UCB algorithms. We then discuss the latter case in which $k' \in \mathcal{K}^{(g(k))} \cap \mathcal{K}_0$ is pulled. If $\hat{n}_t(k') \ge \frac{\alpha \log t}{2\Delta^2_0}$, for any non-target agent $m \notin \mathcal{G}_0$, there always exists $k'' \in \mathcal{K}^{(m)} \setminus \mathcal{K}_0$ such that
\begin{equation}
    UCB_t(k') = \hat{\mu}_t(k') + \sqrt{\frac{\alpha \log t}{2 \hat{n}_t(k')}} \le \mu(k'') - \Delta_0 +  \sqrt{\frac{\alpha \log t}{2 \hat{n}_t(k')}} \le \mu(k'') \le UCB_t(k''),
\end{equation}
where the first inequality is due to our attack design in \cref{eq:OA_attack}. As a result, non-target agents will not pull $k'$ anymore. \cref{eq:OA_attack} also ensures the number of pulls by target agents after $\hat{n}_t(k') \ge \frac{\alpha \log t}{2\Delta^2_0}$ is bounded by $\frac{\alpha \log t}{2\Delta^2_0}$ (similar to the proof of \Cref{lemma:homo-time}). Thus, the total number of arm pulls for all $k' \in \mathcal{K}^{(g(k))} \cap \mathcal{K}_0$ from agent $g(k)$ is upper bounded by $|\mathcal{K}^{(g(k))} \cap \mathcal{K}_0| \cdot \frac{2\alpha \log t}{2\Delta_0^2} = c_{k,3} \log t$. Then for every $k\in \mathcal{K}_0$, we want 
\begin{equation}
    \hat{n}_t(k_0^{(g(k))}) \ge t - c_{k,1}\log t - c_{k,3} \log t \ge c_{k,1} \log t.
\end{equation}
With a feasible $T_0$ such that 
\begin{equation}
        \frac{T_0}{\log T_0} \ge \max_{k \in \mathcal{K}_0} c_{k},
\end{equation}
for any $t > T_0$, arms in $\mathcal{K}_0$ will only be pulled by target agents in $\mathcal{G}_0$, and these agents will conduct attacks according to \cref{eq:OA_attack}. When $t > T_0$, for every $k \in \mathcal{K}_0$, we can follow the same steps in the proof of \Cref{lemma:homo-time} and obtain
\begin{equation}
    \hat{n}_t(k) \le \frac{\alpha \log t}{2\Delta_0^2}.
\end{equation}
Based on this, the cost upper bound in \Cref{thm:OA} can be easily derived by following the same steps in the proof of  \Cref{lemma:homo-cost}, which concludes the proof.
\end{proof}

\subsection{Proof of \texorpdfstring{\Cref{thm:LTA}}{T1}}\label{lta-proof}
\begin{proof}
We first prove that at the end of the learning stage, \Cref{alg:LTA} can learn the correct mean reward ranking of all arms. For every arm $k \in \mathcal{K}$, we have $\hat{n}_t(k) \ge L = \ceil*{\frac{2\log (2K/\delta)}{\Delta^2_{\min}}}$. By Hoeffding’s inequality,
\begin{equation}
    |\hat{\mu}_t(k) - \mu(k)| \le \frac{\Delta_{\min}}{2}
\end{equation}
with probability $1-\delta$, which indicates that sorting all arms according to $\hat{\mu}_t(k)$ will give the correct ranking.

We then consider the attack cost incurred during the learning stage. As discussed in \Cref{subsec:learning-then-attack}, the learning stage ends after no more than \(\frac{KL}{cM}\) rounds. Since the attack value per round is upper bound by $\Delta(1,K)+\beta(1)+b$, the total attack cost of $M$ agents is bounded by \(\frac{MKL(\Delta(1,K)+\beta(1)+b)}{cM}\), where $b$ is the upper bound of the mean rewards. 
In addition, \Cref{alg:LTA} also needs to pay a significant cost $\gamma_t^{(m)}$ to alter the unbiased empirical mean of $k \in \mathcal{K}_0$ for the first time that $k$ is attacked by $m \in \mathcal{G}_0$ during the second stage. This is due to the large value of $\hat{n}_{t-1}(k)$ when calculating $\hat{\mu}_{t}(k)$, which necessities a relatively large $\gamma_t^{(m)}$ to ensure \cref{eq:OA_attack}. Since this cost can still be upper bounded by \(\frac{KL(\Delta(1,K)+\beta(1)+b)}{c}\), the total additional cost induced by the learning stage is \(\frac{2KL(\Delta(1,K)+\beta(1)+b)}{c}\), which appears as the first term of $C(T)$ in \Cref{thm:LTA}. Notice that the learning stage of \Cref{alg:LTA} directly ensures that  $k_0^{(g(k))}$ has been sufficiently pulled for every $k\in \mathcal{K}_0$. Thus, the cost of the attack stage is the same as that in \Cref{alg:OA}, which concludes the proof.

\end{proof}

\section{Additional Results in Homogeneous Settings}
\subsection{Attacks against UCB-TCOM}
As mentioned in Section \ref{sec:homo}, UCB-TCOM stands as the state-of-art algorithm in homogeneous settings. It boasts near-optimal regret, with communication costs limited to just $O(\log\log T)$. The key ideas that help UCB-TCOM algorithm to decrease the communication cost are: first, communications occur only when there are sufficient local samples for the agents; second, information about the optimal arm is not directly broadcast. In this section, we demonstrate that our attack strategy is also effective against the UCB-TCOM algorithm.

To begin with, we first discuss why we need to slightly change \cref{equ:IC-inequality}. This adjustment is required due to the consideration of delayed information. Under the UCB-TCOM strategy, all $M$ agents simultaneously select an arm \(k\) for multiple consecutive rounds until the number of samples \(\hat{n}_t(k)\) exceeds a predefined threshold. We refer to these consecutive rounds as a \emph{phase}. The agents share their local information and update $\hat{\mu}_t(k)$, $\hat{n}_t(k)$, and UCB values at the end of each phase. As a result, it becomes essential to compute the attack value carefully in each round, accounting for the delayed counters. Assume arm $k_t\upbra{m} = k \neq K$ for all agents $m \in \mathcal{M}$, where $t$ belongs to a phase from round $s+1$ to $r$, and denote the last phase that arm $k$ is selected ends at as round $t'$. We provide the condition in \cref{equ:IC-inequality} again here:
\begin{equation}\label{eq:TCOM-1}
	\hat{\mu}_{t}(k) \leq \hat{\mu}_t(K) - 2\beta(\hat{n}_t(K)) - \Delta_0,
\end{equation}
while $\hat{\mu}_t(K) = \hat{\mu}_s(K)$ and $\hat{n}_t(K) = \hat{n}_s(K)$ and they can be computed in round $t$. In addition, we can compute the true value of $\hat{\mu}_t(k)$ even if the agents will not update this value. Assume that we only attack agent $m \in \mathcal{M}$. Then, we can compute the attack values:
\begin{equation}\label{eq:TCOM-2}
	\begin{aligned}
		 \gamma\upbra{m}(t) = &  (\hat{\mu}_{t'}(k)\hat{n}_{t'}(k) + \sum_{h=s+1}^t\sum_{m'=1}^MX_{h}\upbra{m',0}(k) - \sum_{h=s+1}^{t-1}\gamma\upbra{m}(h)        \\
		 & - (\hat{n}_{t'}(k)+(t-s)M)          (\hat{\mu}_t(K)-2\beta(\hat{n}_t(K))-\Delta_0))_+,
	\end{aligned}
\end{equation}
where $(x)_+$ represents the maximum of $x$ and $0$. It is worth noting that \cref{eq:TCOM-2} handles samples from both the previous phase and the current one separately. The latter requires special consideration due to the delayed updates. The entire process is outlined in \cref{alg:TCOM}.
\begin{algorithm}[t]
	\caption{Attack against UCB-TCOM (Agent \(m\))}\label{alg:TCOM}
	\begin{algorithmic}[1]
		\State \textbf{Initialization}: $\hat{\mu}_t(k) = \hat{n}_t(k) = 0$ for all $k \in [K]$
		\For{$t = 1,2,3,\dots,T$}
		\State Attacker observes that agent selects $k_t\upbra{m}$ by UCB
		\State Environment reveals reward $X_t\upbra{m,0}(k_t\upbra{m})$
		\If {$k_t\upbra{m} \neq K$}
		\State Attacker manipulates reward $X_t\upbra{m}(k_t\upbra{m}) = X_t\upbra{m,0}(k_t\upbra{m}) - \gamma\upbra{m}(t)$ according to \cref{eq:TCOM-2}
		\EndIf
		\EndFor
	\end{algorithmic}
\end{algorithm}

\begin{theorem}\label{thm:TCOM}
	Suppose $T > K$, set the parameters of UCB-TCOM as $\beta > 1$ and $\delta < 1/2$. With probability at least $1 - \delta$, \Cref{alg:TCOM} misguides the UCB-TCOM algorithm to choose the target arm $K$ at least
	$MT - (K-1)\left(\frac{2\beta}{\Delta_0^2} \log T\right)$ rounds, using a a cumulative attack cost at most
	\begin{equation*}\textstyle
		\begin{aligned}
			C(T) \le
\left(\frac{2\beta}{\Delta_0^2} \log T\right)\sum_{k<K}(\Delta(k,K) + \Delta_0) + \frac{8(K-1)\sigma^2}{\Delta_0^2}\sqrt{\beta\log T\log\frac{4K\beta^2\pi^2(\log T)^2}{3\delta\Delta_0^4}}.
		\end{aligned}
	\end{equation*}
\end{theorem}

\begin{proof}
The proof is similar to \cref{homo-proof} and we use the same notations there. However, some steps should be modified carefully.

\begin{lemma}\label{TCOM-lemma1}
Assume event $E$ holds and $\delta \le 1/2$. For any $k \neq K$ and any $t > K$, we have
\begin{equation}
    \hat{n}_t(k) \le \min\{\beta \hat{n}_t(K), \frac{2\beta \log t}{\Delta_0^2} \}
\end{equation}
\end{lemma}

\begin{proof}
Fix some $t > K$, which satisfies $k_t\upbra{m} = k \neq K$ for all $m \in \mathcal{M}$, and $t$ is in a phase from round $s+1$ to $r$. Also, denote the last phase arm $k$ is pulled is ended at round $t'$. In round $t'$, we have the following inequality by our attack design:
\begin{equation}\label{eq:TCOM-lemma1-1}
    \hat{\mu}_{t'}(k) \le \hat{\mu}_{t'}(K) - 2\beta(\hat{n}_{t'}(K)) -\Delta_0.
\end{equation}
On the other hand, arm $k$ is selected in round $s+1$ because it has the highest UCB in round $s$:
\begin{equation}\label{eq:TCOM-lemma1-2}
    \hat{\mu}_s(k) + \sqrt{\frac{2\log s}{\hat{n}_s(k)}} \ge \hat{\mu}_s(K) + \sqrt{\frac{2\log s}{\hat{n}_s(K)}}.
\end{equation}
Note that $\hat{\mu}_{t'}(k) = \hat{\mu}_s(k)$. Substituting \cref{eq:TCOM-lemma1-1} into \cref{eq:TCOM-lemma1-2}, we have:
\begin{equation}\label{eq:TCOM-lemma1-3}
    \begin{aligned}
        \sqrt{\frac{2\log s}{\hat{n}_s(k)}} - \sqrt{\frac{2\log s}{\hat{n}_s(K)}} & \ge \hat{\mu}_s(K) - \hat{\mu}_s(k) \\
        & \ge \hat{\mu}_s(K) - (\hat{\mu}_{t'}(K) - 2\beta(\hat{n}_{t'}(K)) -\Delta_0)  \\
        & \ge \Delta_0 \\
        & > 0,
    \end{aligned}
\end{equation}
where the third inequality is due to the monotonically decreasing property of $\beta()$. Therefore, 
\begin{equation}\label{eq:TCOM-lemma1-4}
    \hat{n}_t(K) = \hat{n}_s(K) \ge \hat{n}_s(k) = \frac{1}{\beta}\hat{n}_r(k) \ge \frac{1}{\beta}\hat{n}_t(k).
\end{equation}
In addition, as the bonus term is non-negative, we have:
\begin{equation}\label{eq:TCOM-lemma1-5}
    \hat{n}_t(k) \le \hat{n}_r(k) = \beta \hat{n}_s(k) \le \frac{2\beta\log s}{\Delta_0^2} \le \frac{2\beta\log t}{\Delta_0^2}.
\end{equation}
\end{proof}

\begin{lemma}\label{TCOM-lemma2}
Assume event $E$ holds and $\delta \le 1/2$. For any $k \neq K$ and any $t > K$, we have
\begin{equation}
    \sum_{h \in \tau(t,k)} \gamma\upbra{m}(h) \le \hat{n}_{t}(k)(\Delta(k, K) + \Delta_0 + 3 \beta(\hat{n}_t(K)) + \beta(\hat{n}_t(k)))
\end{equation}
\end{lemma}

\begin{proof}
Note that although agents do not update their counters until each phase is over, the attacker does have the latest information thus it can maintain the latest $\hat{\mu}_{t}(k)$ and $\hat{n}_{t}(k)$ in each round $t$ even if it is not the last round of a phase. \cref{eq:TCOM-2} can be written in this form:
\begin{equation}\label{eq:TCOM-lemma2-1}
    \begin{aligned}
         \gamma\upbra{m}(t) = &  (\hat{\mu}_{t'}\upbra{0}(k)\hat{n}_{t'}(k) + \sum_{h=s+1}^t\sum_{m'=1}^MX_{h}\upbra{m',0}(k) - \sum_{h\in \tau(t',k)}\gamma\upbra{m}(h) - \sum_{h=s+1}^{t-1}\gamma\upbra{m}(h)        \\
		 & - (\hat{n}_{t'}(k)+(t-s)M)          (\hat{\mu}_r(K)-2\beta(\hat{n}_r(K))-\Delta_0))_+,
    \end{aligned}
\end{equation}
where $\hat{\mu}_{t'}\upbra{0}(k)$ is the global pre-attack empirical mean of arm $k$ up to round $t'$. Also, as in \cref{homo-proof}, we only consider the round $t$ such that  $\gamma\upbra{m}(t) > 0$. Therefore, we have:
\begin{equation}\label{eq:TCOM-lemma2-2}
    \begin{aligned}
        \sum_{h\in \tau(t,k)}\gamma\upbra{m}(h) & = \hat{\mu}_{t'}\upbra{0}(k)\hat{n}_{t'}(k) + \sum_{h=s+1}^t\sum_{m'=1}^MX_{h}\upbra{m',0}(k) - (\hat{n}_{t'}(k)+(t-s)M)          (\hat{\mu}_r(K)-2\beta(\hat{n}_r(K))-\Delta_0) \\
        & = \hat{n}_t(k)(\hat{\mu}_t^0(k) - (\hat{\mu}_t(K)-2\beta(\hat{n}_t(K))-\Delta_0)) \\
        & \le \hat{n}_t(k)(\Delta(k,K) + \Delta_0 + 3 \beta(\hat{n}_t(K)) + \beta(\hat{n}_t(k))),
    \end{aligned}
\end{equation}
where the last inequality is due to the event $E$.
\end{proof}

With \cref{TCOM-lemma1}, we can easily get that the target arm $K$ is selected for at least $MT - (K-1)(\frac{2\beta}{\Delta_0^2}\log T)$ times. For the cumulative cost, we use \cref{TCOM-lemma2}, and sum over all non-target arms. Also, it is easy to get $\beta(\hat{n}_t(K)) \le \beta(\frac{1}{\beta}\hat{n}_t(k))$ as $\beta()$ is a monotonically decreasing function. Therefore,
\begin{equation}\label{equ:TCOM_Thmproof}
    \begin{aligned}
        \sum_{t = 1}^{T} \gamma\upbra{m}(t) & \le \sum_{k = 1}^{K-1}\hat{n}_t(k)(\Delta(k,K)+\Delta_0) + 4\beta\sum_{k = 1}^{K-1} \hat{n}_t(k) \beta(\frac{1}{\beta}\hat{n}_t(k))\\
        & \le \left(\frac{2\beta}{\Delta_0^2} \log T\right)\sum_{k<K}(\Delta(k,K) + \Delta_0)   + \frac{8(K-1)\beta^2\sigma}{\Delta_0^2}\sqrt{\log T\log\frac{4K\pi^2(\log T)^2}{3\delta\Delta_0^4}}.
    \end{aligned}
\end{equation}
\end{proof}

\subsection{Attacks against Leader-follower Algorithm}
In contrast to fully distributed algorithms, there exists a server, or leader (agent) in leader-follower algorithms, which has a pivotal role in exploration. On the other hand, the followers always undertake exploitation. In this section, we consider attacks on a representative leader-follower algorithm, the DPE2 algorithm, proposed by ~\citep{wang2020optimal}. We show that our attack algorithm against fully distributed algorithms can be extended to these leader-follower algorithms.

We first introduce some new notations to differentiate between leader-follower algorithms and fully distributed algorithms. Without loss of generality, let agent $1$ be the leader of the system. Let $\hat{N}_t(k)$ denote the times arm $k$ is selected up to time slot $t$, and $\hat{V}_t(k)$ be the post-attack empirical mean associated with $\hat{N}_t(k)$. For the ease of presentation, we consider the UCB1 induces instead of the KL-UCB induces, and define $D_t(k) := \hat{V}_t(k) + \sqrt{\frac{\alpha\log t}{2\hat{N}_t(k)}}$ as the UCB. The leader explores different arms by maintaining a list $C(t)$ which contains the suboptimal arms whose upper bounds are larger than the empirical mean of what it considers to be the optimal arm. Similar to the UCB-TCOM algorithm, the information is not updated immediately after each round. We also define the phase in the process. When $C(s-1)=\emptyset$ and $C(s)\neq\emptyset$, we say the phase begins at round $s$; and when $C(r-1)\neq\emptyset$ and $C(r)=\emptyset$, we say the phase ends at round $r-1$. As the design of DPE2, the information of all arms will be updated in round $r$. We then introduce our attack algorithm.

As the followers always select the arm which the leader considers as the best, we only need to misguide the leader to regard the target arm $K$ as the optimal arm. Therefore, we only need to attack the leader. Assume $k_t\upbra{1} =k \neq K$, and $t$ belongs to a phase from round $s$ to $r-1$. Our attacks make sure:
\begin{equation}\label{eq:l-f-attack}
    \hat{V}_r(k) \le \hat{V}_s(K) - 2\beta(\hat{N}_s(K)) - \Delta_0,
\end{equation}
where $\hat{V}_r(k) = \frac{\hat{V}_s(k)\hat{N}_s(k) + X_t\upbra{1,0}(k)-\gamma(t)}{\hat{N}_r(k)}$, and the attack value is $\gamma(t)$. The details are described in \cref{alg:DPE2}.

\begin{algorithm}[t]
	\caption{Attack against DPE2 (Leader)}\label{alg:DPE2}
	\begin{algorithmic}[1]
		\State \textbf{Initialization}: $\hat{V}_t(k) = \hat{N}_t(k) = D_t(k) = 0$ for all $k \in [K]$, $C(t)=\emptyset$
		\For{$t = 1,2,3,\dots,T$}
		\State Attacker observes that agent selects $k_t\upbra{1}$
		\State Environment reveals reward $X_t\upbra{1,0}(k_t\upbra{1})$
		\If {$k_t\upbra{1} \neq K$}
		\State Attacker manipulates reward $X_t\upbra{1}(k_t\upbra{1}) = X_t\upbra{1,0}(k_t\upbra{1}) - \gamma(t)$ according to \cref{eq:l-f-attack}
		\EndIf
		\EndFor
	\end{algorithmic}
\end{algorithm}

\begin{theorem}\label{thm:DPE2}
	Suppose $T > T_0$, and $\delta < 1/2$. With probability at least $1 - \delta$, \Cref{alg:DPE2} misguides the DPE2 algorithm to choose the target arm $K$ at least
	$M(T - K) - (K-1)\left(\frac{\alpha}{2\Delta_0^2} \log T + 1\right)$ rounds, using a a cumulative attack cost at most
	\begin{equation*}\textstyle
		\begin{aligned}
			C(T) \le
			 & \left(\frac{\alpha}{2\Delta_0^2} \log T + 1\right)\sum_{k<K}(\Delta(k,K) + \Delta_0)                               \\
			 & + 4(K-1)\sigma\sqrt{2(\frac{\alpha}{2\Delta_0^2} \log T + 1)\log(\frac{K\pi^2}{3\delta}(\frac{\alpha}{2\Delta_0^2} \log T + 1)^2)},
		\end{aligned}
	\end{equation*}
    where $T_0/\log(T_0) = K\lceil\frac{\alpha}{2\Delta_0^2}+1\rceil$.
\end{theorem}

\begin{proof}
The proof is similar to \cref{homo-proof} as well.

\begin{lemma}\label{DPE2-lemma1}
Assume event $E$ holds and $\delta \le 1/2$. For any $k \neq K$ and any $t > T_0$, we have
\begin{equation}\nonumber
    \hat{N}_t(k) \le \frac{\alpha}{2\Delta_0^2}\log t + 1
\end{equation}
\end{lemma}

\begin{proof}
Fix some $t > T_0 \ge K$, which satisfies $k_t\upbra{1} = k \neq K$, and $t$ is in a phase from round $s$ to $r-1$. Also, assume the last phase arm $k$ is pulled from round $s'$ to $r'-1$. In round $r'$, we have the following inequality by the design of our attacks:
\begin{equation}\label{eq:DPE2-lemma1-1}
    \hat{V}_{r'}(k) \le \hat{V}_{s'}(K) - 2\beta(\hat{N}_{s'}(K)) -\Delta_0.
\end{equation}
On the other hand, arm $k$ is selected in round $s$ because the following inequality holds in round $s$:
\begin{equation}\label{eq:DPE2-lemma1-2}
    \hat{V}_s(k) + \sqrt{\frac{\alpha\log s}{\hat{N}_s(k)}} \ge \hat{V}_s(K).
\end{equation}
Note that $\hat{V}_{s}(k) = \hat{V}_{r'}(k)$. Substituting \cref{eq:DPE2-lemma1-1} into \cref{eq:DPE2-lemma1-2}, we have:
\begin{equation}\label{eq:DPE2-lemma1-3}
    \begin{aligned}
        \sqrt{\frac{\alpha\log s}{\hat{N}_s(k)}} & \ge \hat{V}_s(K) - \hat{V}_s(k) \\
        & \ge \hat{V}_s(K) - (\hat{V}_{s'}(K) - 2\beta(\hat{N}_{s'}(K)) -\Delta_0)  \\
        & \ge \Delta_0,
    \end{aligned}
\end{equation}
where the third inequality is due to the monotonically decreasing property of $\beta$. Therefore, this ends the proof as $\hat{N}_t(k) \le \hat{N}_s(k) + 1$.
\end{proof}

\begin{lemma}\label{DPE2-lemma2}
Assume event $E$ holds and $\delta \le 1/2$. For any $k \neq K$ and any $t > K$, we have
\begin{equation}
    \sum_{h \in \tau(t,k)} \gamma(h) \le \hat{N}_{t}(k)(\Delta(k, K) + \Delta_0 + 3 \beta(\hat{N}_t(K)) + \beta(\hat{N}_t(k)))
\end{equation}
\end{lemma}

\begin{proof}
Similar to \cref{eq:TCOM-lemma2-1}, the attack value can be written in this form:
\begin{equation}\label{eq:DPE2-lemma2-1}
    \begin{aligned}
         \gamma(t) = &  \left(\hat{V}_{s}\upbra{0}(k)\hat{N}_{s}(k) + X_{t}\upbra{1,0}(k) - \sum_{h\in \tau(s,k)}\gamma(h)  - (\hat{N}_{s}(k)+1)          (\hat{V}_s(K)-2\beta(\hat{N}_s(K))-\Delta_0)\right)_+,
    \end{aligned}
\end{equation}
where $\hat{V}_{s}\upbra{0}(k)$ is the global pre-attack empirical mean of arm $k$ up to round $s$ (for the leader). Also, as in \cref{homo-proof}, we only consider the round $t$ such that  $\gamma(t) > 0$. Therefore, we have:
\begin{equation}\label{eq:DPE2-lemma2-2}
    \begin{aligned}
        \sum_{h\in \tau(t,k)}\gamma(h) & = \hat{V}_{s}\upbra{0}(k)\hat{N}_{s}(k) + X_{t}\upbra{1,0}(k) - (\hat{N}_{s}(k)+1)          (\hat{V}_s(K)-2\beta(\hat{N}_s(K))-\Delta_0) \\
        & = \hat{N}_t(k)(\hat{V}_t^0(k) - (\hat{V}_t(K)-2\beta(\hat{N}_t(K))-\Delta_0)) \\
        & \le \hat{N}_t(k)(\Delta(k,K) + \Delta_0 + 3 \beta(\hat{N}_t(K)) + \beta(\hat{N}_t(k))),
    \end{aligned}
\end{equation}
where the last inequality is due to the event $E$.
\end{proof}

With \cref{DPE2-lemma1}, we can easily get that the target arm $K$ is selected for at least $M(T - K) - (K-1)\left(\frac{\alpha}{2\Delta_0^2} \log T + 1\right)$ times because in the beginning, followers randomly select an arm to pull, and after that, $K$ will be the optimal arm for the leader after each phase, so that followers won't select arms other than $K$. As for the cumulative cost, we introduce \cref{DPE2-lemma2}, and sum over all non-target arms. Also, it is easy to get $\beta(\hat{N}_t(K)) \le \beta(\hat{N}_t(k))$ as $\beta$ is a monotonically decreasing function and $T> T_0$, which means $\hat{N}_t(k) \le \hat{N}_t(K)$ holds for any $k \neq K$. Therefore,
\begin{equation}\label{equ:DPE2_Thmproof}
    \begin{aligned}
        \sum_{t = 1}^{T} \gamma(t) & \le \sum_{k = 1}^{K-1}\hat{N}_t(k)(\Delta(k,K)+\Delta_0) + 4\beta\sum_{k = 1}^{K-1} \hat{N}_t(k) \beta(\hat{N}_t(k))\\
        & \le \left(\frac{\alpha}{2\Delta_0^2} \log T + 1\right)\sum_{k<K}(\Delta(k,K) + \Delta_0) \\
        & + 4(K-1)\sigma\sqrt{2(\frac{\alpha}{2\Delta_0^2} \log T + 1)\log(\frac{K\pi^2}{3\delta}(\frac{\alpha}{2\Delta_0^2} \log T + 1)^2)}.
    \end{aligned}
\end{equation}
\end{proof}

\subsection{Additional Experiments}
\begin{figure}[t]
    \centering
    \begin{subfigure}[b]{0.28\textwidth}
        \centering
        \includegraphics[width=\textwidth]{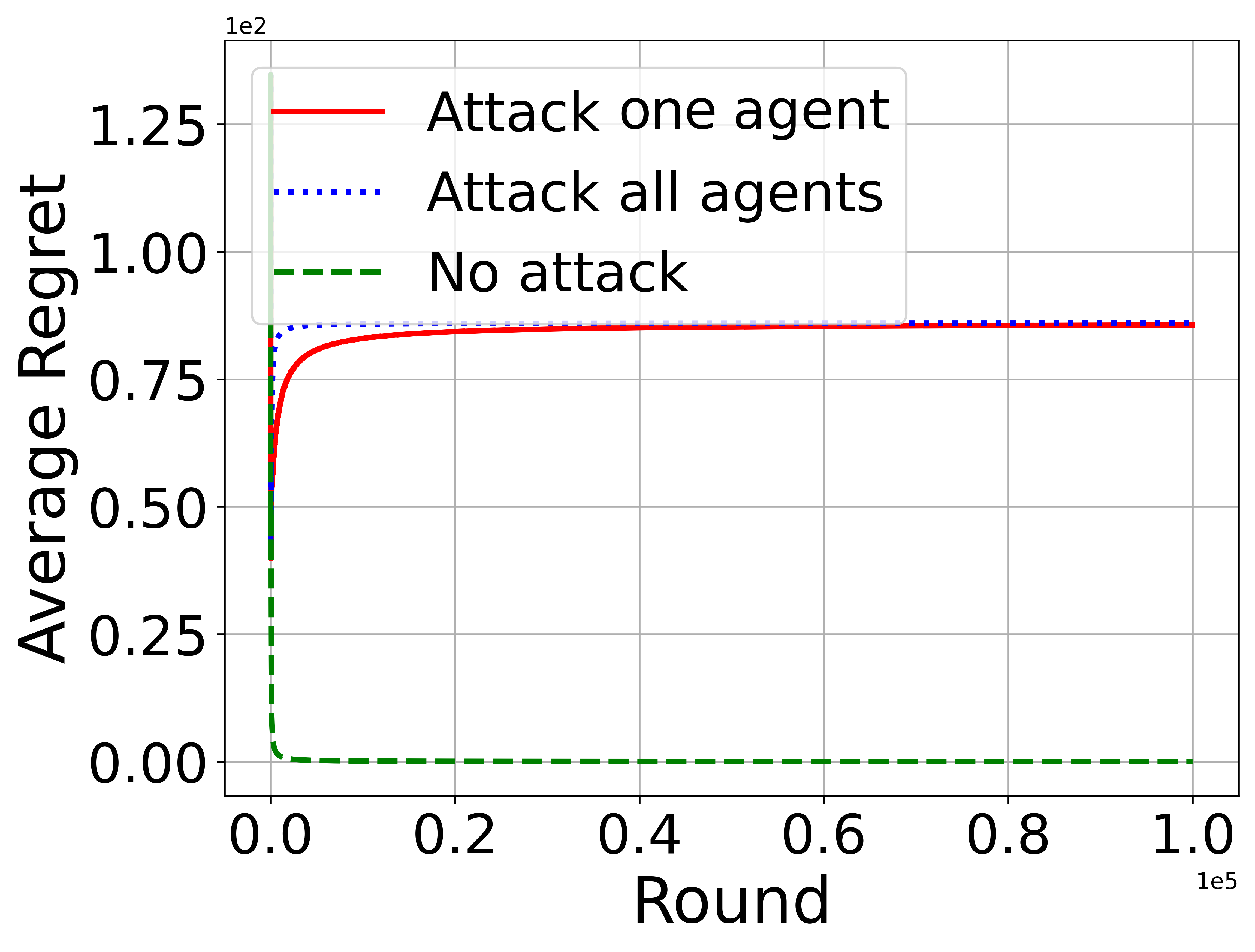}
        \caption{Average Regrets}
        \label{fig:IC-Regret}
    \end{subfigure}
    \begin{subfigure}[b]{0.28\textwidth}
        \centering
        \includegraphics[width=\textwidth]{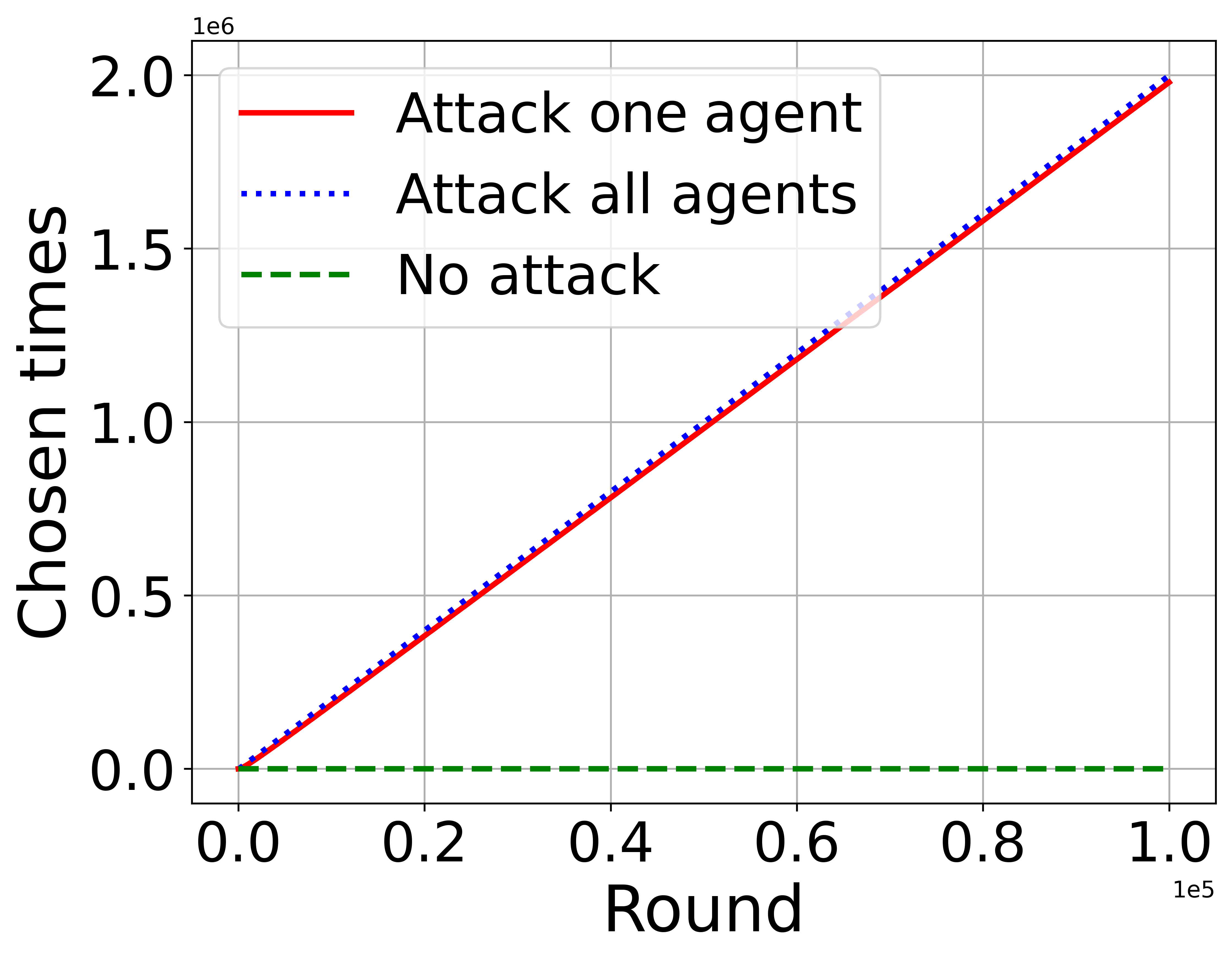}
        \caption{Chosen times}
        \label{fig:IC-Times}
    \end{subfigure}
    \begin{subfigure}[b]{0.28\textwidth}
        \centering
        \includegraphics[width=\textwidth]{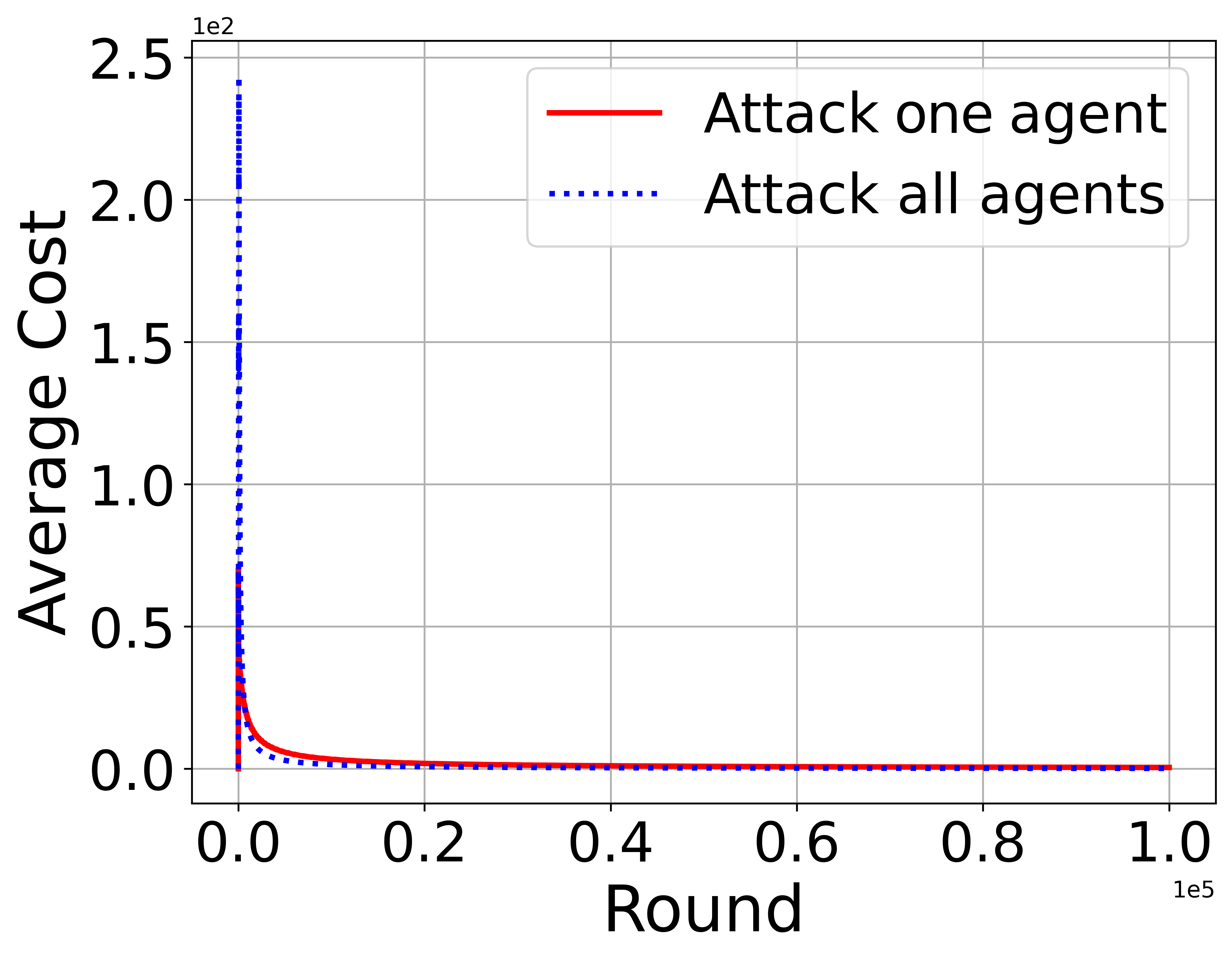}
        \caption{Average Costs}
        \label{fig:IC-Cost}
    \end{subfigure}
    \caption{Attacks against UCB.}
    \label{fig:UCB-IC}
\end{figure}
In this section, we show the experimental results of attacks against the CO-UCB algorithm in homogeneous settings. We set $T = 100,000,K = 20, M =20$. The distributions of arms are the same as those in \cref{sec:exp}. CO-UCB takes $\alpha = 4$, and attack parameters are set to $\Delta_0 = 0.1$, and $\delta = 0.1$. We compare our algorithm, which only attacks one agent, with two baselines: the first one attacks all agents using attack values computed by \cref{equ:IC-inequality} for each agent; the second one is the original CO-UCB algorithm without attacks. Each experiment was repeated for $10$ times.


In Figures \cref{fig:IC-Regret} and \cref{fig:IC-Times}, the CO-UCB algorithm displays sublinear regret; the curve showcasing its regret gradually approaches $0$, and the algorithm seldom opts for the suboptimal target arm $K$. However, both attack strategies successfully misguide the CO-UCB algorithm, leading it to consistently select our target arm, as shown in \cref{fig:IC-Times}. This causes linear regrets, as highlighted in \cref{fig:IC-Regret}. 
Notably, the divergence between these attack algorithms in both figures is minimal, suggesting that even though our attack is designed to target just one agent, it is almost as effective as a strategy targeting all agents. 
Looking at \cref{fig:IC-Cost}, the cumulative attack costs for both strategies are nearly indistinguishable, converging towards $0$. This points to sublinear costs for both. Such a finding amplifies the effectiveness of our attack design: despite focusing on a single agent, the cumulative cost is nearly identical to an approach targeting all $20$ agents.

\end{document}